\documentclass{article}

\PassOptionsToPackage{numbers,compress}{natbib}

\usepackage[preprint]{neurips_2022}

\usepackage{graphicx}

\usepackage[utf8]{inputenc} 
\usepackage[T1]{fontenc}    
\usepackage[pagebackref=true,breaklinks=true,colorlinks,bookmarks=false]{hyperref} 
\usepackage{url}            
\usepackage{booktabs}       
\usepackage{amsfonts}       
\usepackage{nicefrac}       
\usepackage{microtype}  
\usepackage{xcolor}

\usepackage{caption}
\usepackage{subcaption}
\usepackage{wrapfig}
\usepackage{float}

\usepackage{algorithm}
\usepackage{algpseudocode}

\def\onedot{.}
\def\eg{{\em e.g}\onedot} 
\def\ie{{\em i.e}\onedot}

\usepackage{amsmath}
\usepackage{amssymb}
\usepackage{mathtools}
\usepackage{amsthm}
\usepackage{bm}
\newtheorem{theorem}{Theorem}
\newtheorem{lemma}[theorem]{Lemma}
\newtheorem{remark}{Remark}

\usepackage{diagbox}
\usepackage{multirow}

\title{Accelerating Score-based Generative Models for High-Resolution Image Synthesis}

\author{%
  Hengyuan Ma$^1$
  \quad
  Li Zhang$^1$\thanks{Li Zhang (lizhangfd@fudan.edu.cn) is the corresponding author with School of Data Science, Fudan University.}
  \quad
  Xiatian Zhu$^2$
  \quad
  Jingfeng Zhang$^3$
  \quad
  Jianfeng Feng$^1$ 
  \vspace{.5em} 
  \\
  $^1$Fudan University
  \qquad
  $^2$University of Surrey
  \qquad
  $^3$RIKEN
  \vspace{.5em} 
  \\
  \url{https://fudan-zvg.github.io/TDAS}
}

\begin{document}

\maketitle

\begin{abstract}
Score-based generative models (SGMs) have recently emerged as a promising class of generative models.
The key idea is to produce high-quality images by recurrently adding Gaussian noises and gradients to a Gaussian sample until converging to the target distribution, a.k.a. the diffusion sampling.
To ensure stability of convergence in sampling and generation quality, however, this sequential sampling process has to take a {\em small step size} and many sampling iterations (\eg, 2000).
Several acceleration methods have been proposed
with focus on low-resolution generation. 
In this work, we consider the acceleration of high-resolution 
generation with SGMs, a more challenging yet more important problem. 
We prove theoretically that this {slow convergence} drawback is primarily due to the ignorance of the target distribution.
Further, we introduce a novel {\bf\em Target Distribution Aware Sampling} (TDAS) method by leveraging the structural priors in space and frequency domains.
Extensive experiments on CIFAR-10, CelebA, LSUN, and FFHQ datasets validate that TDAS can consistently accelerate state-of-the-art SGMs, particularly on more challenging high resolution ($1024\times 1024$) image generation tasks by up to $18.4\times$, whilst largely maintaining the synthesis quality.
With fewer sampling iterations, TDAS can still generate good quality images. In contrast, the existing methods degrade drastically or even fails completely.
\end{abstract}

\section{Introduction}
\begin{figure}[ht]
    \hspace{-11ex}
    \includegraphics[scale=0.24]{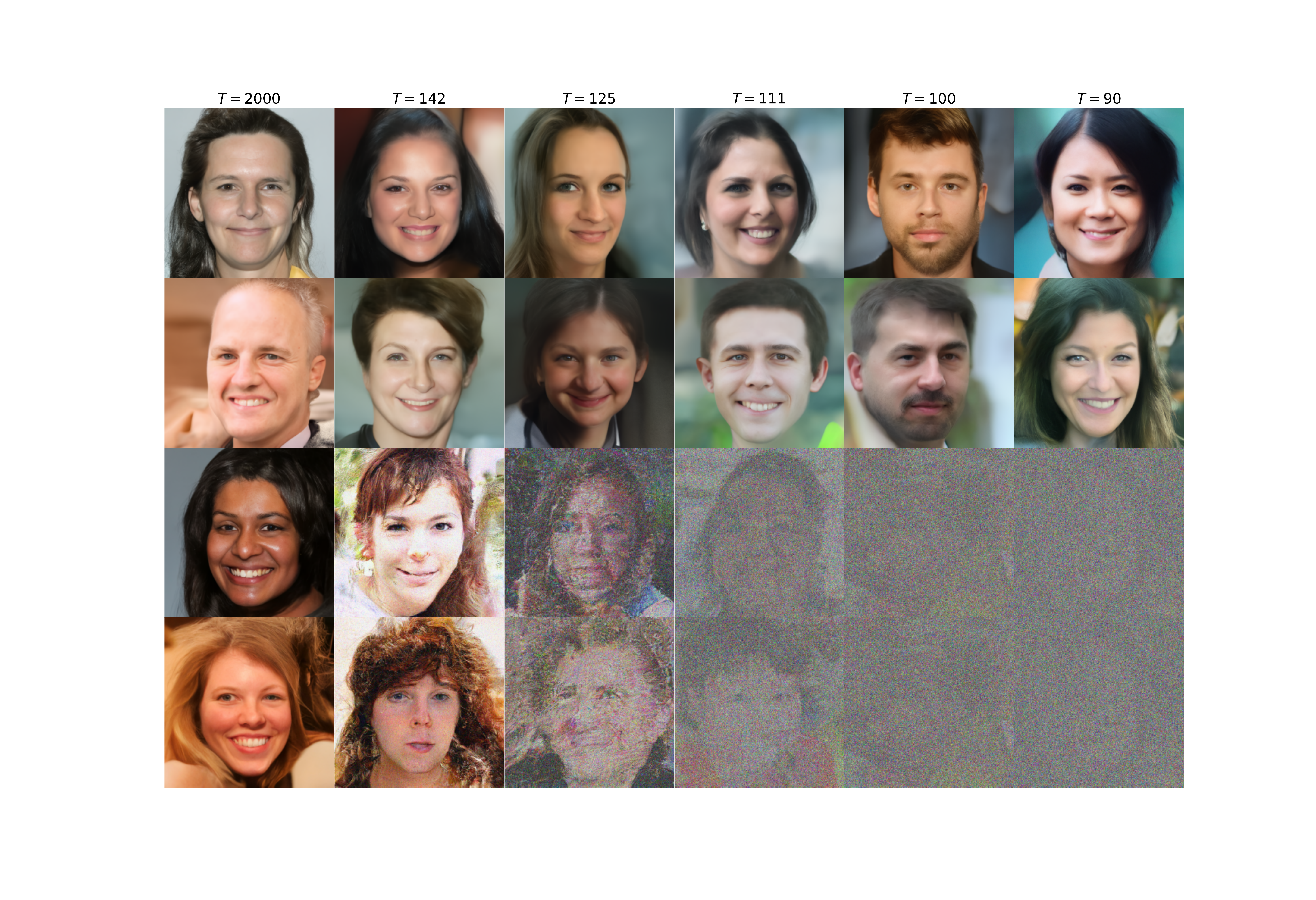}
    \vspace{-9ex}
    		\caption{High resolution facial images (FFHQ~\cite{karras2019style} $1024 \times 1024$) generated by the proposed TDAS (top two rows) and original (vanilla) sampling (bottom two rows) under a variety of sampling iterations $T$.
    		It is evident that 
with less sampling iterations, 
TDAS can still generate good quality of images
while the existing counterpart degrades drastically 
or even fails completely.
    		In terms of running speed for generating a batch of 8 images, TDAS reduces the time cost from 1758.6 seconds ($T=2000$) 
    		to 95.8 seconds ($T=100$) on one NVIDIA RTX 3090 GPU,
    		which delivers $18.4\times$ acceleration.
    		{\em SGM}: NCSN++ \cite{song2020score}.
    		More samples in Appendix.
    		}\label{fig: demo_face}
\end{figure}

As an alternative framework to generative adversarial networks (GANs)~\cite{goodfellow2014generative}, score-based generative models (SGMs) have recently demonstrated excellent abilities in data synthesis (\eg, high resolution images) with easier optimization \cite{de2021diffusion} and richer generative diversity \cite{Xiao2021Tackling123}.
Starting from a sample initialized with a Gaussian distribution, a SGM produces a target sample by recurrently adding a learned gradient and a Gaussian noise to this sample, namely {\em the diffusion process}.

Compared to GANs, a significant drawback of existing SGMs is {\em drastically slower synthesis} due to the need for many (\eg, 2000) iterations of sampling and gradient computation.
This is due to an indispensable constraint of using small step sizes, otherwise the SGMs will fail to converge subject to low stability.
We conjecture that this limitation can be alleviated by increasing the similarity between the target data distribution and diffusion related (sample initialization and noise sampling) distributions.
Further, we prove that leveraging the target distribution for sample initialization and noise sampling helps to accelerate the diffusion process until convergence significantly (c.f. Theorem~\ref{thm: deviation}).
This is intuitive as both the starting point and {added noises} are constrained to be relevant w.r.t. the target distribution, hence minimizing the inefficient fluctuations over the diffusion process.
Through orthogonal invariance analysis (Theorem~\ref{Thm: transform}), we prove that the diffusion related (sample initialization and noise sampling) distributions of the existing sampling methods are insensitive to the ordering of coordinates (\eg, the location of pixels in images).
Under this theory, we introduce a model agnostic {\bf \em Target Distribution Aware Sampling} (TDAS) approach.
Taking image generation as a showcase, we implement an instantiation of TDAS.
Concretely, we observe that in the frequency domain the amplitude of an average image is mostly concentrated in the low-frequency range;
In the space domain, some types of target images instead present highly consistent structural and geometric similarity (\eg, facial images).
To exploit these structured priors in the SGMs, we propose a simple yet effective {\em space-frequency filter} to regulate sample initialization and noises.

In this work, we make the following {\bf contributions}:
{\bf (1)} 
{We investigate the largely ignored efficiency limitation with existing SGMs. 
To that end, we provide an orthogonal invariance analysis to prove that previous SGMs are insensitive to the order of coordinates.
Further, we identify increasing the similarity between the target data distribution and these diffusion related distributions serves as an underlying contributor for convergence 
(Theorem~\ref{thm: deviation})}.
{\bf (2)}
{
Under this finding, we propose a \textbf{\em Target Distribution Aware Sampling} (TDAS) approach to improving the convergence efficiency of existing SGMs.
}
{\bf (3)} 
{
We further instantiate a specific diffusion process with 
the proposed TDAS for image generation by exploiting the structural distributional priors in both frequency and space domains.}
{\bf (4)} 
{
Extensive experiments demonstrate the advantages of {TDAS} for (NCSN~\cite{song2019generative} and NCSN++~\cite{song2020score}) in generating quality images at a variety of ($32\times 32$ to $1024 \times 1024$) resolutions.
Critically, for more challenging high resolution image generation tasks, TDAS can reduce up to $20\times$ sampling iterations and accelerate the running speed by up to $18.4\times$ whilst largely maintaining the generation quality. 
Particularly, TDAS can generate much better images in smaller sampling number cases (Fig.~\ref{fig: demo_face}).
}
\section{Related work}

\textbf{SGMs. }
Inspired by non-equilibrium statistical physics,
\cite{sohl2015deep} first proposed to destroy the data distribution through a diffusion process slowly,
and learn the backward process to restore the data. 
Later on, \cite{song2019generative} explored score matching for generative models by introducing the noise conditional score network (NCSN). 
\cite{DBLP:conf/nips/0011E20} further scaled
NCSN for higher resolution image generation (\eg, $256\times 256$)
by scaling noises and improving stability with moving average.
Interestingly, \cite{song2020score} summarized all the previous SGMs into a unified framework based on the stochastic differential equation, and proposed the NCSN++ model to generate high-resolution images via numerical SDE solvers for the first time.
\cite{dockhorn2021score} introduced a more complex critically-damped Langevin diffusion (CLD) based SGMs based on Hamiltonian Monte Carlo methods~\cite{neal2011mcmc}, where the diffuison process is also guided by a coupled velocity dynamic.
On the contrary, \cite{jing2022subspace} proposed to restrict the diffusion process on a series of smaller linear subspaces to reduce the computational cost.
As a variant of SGMs, \cite{DBLP:conf/nips/HoJA20} introduced the denoising diffusion probabilistic models (DDPMs) trained by decreasing the variational bound on generative log-likelihood.
\cite{DBLP:conf/icml/NicholD21} improved the training of DDMPs through learning the noise variance and reducing the gradient noise. 
\cite{ho2021cascaded} developed a cascaded DDPM to generate images at increased resolutions.
Commonly, the existing SGMs use isotropic Gaussian distributions for the diffusion sampling. In this work, we prove that this prevents the SGMs from leveraging the structural distributional priors of target data and leads to slow convergence. We further propose a novel TDAS method to solve this fundamental limitation.

\textbf{Accelerating SGMs. }
Recently there are various works proposed on accelerating diffusion sampling.
\cite{Watson2021LearningTE,bao2022analytic} used a dynamic programming algorithm
to find the optimal discrete time schedules to speed up the sampling of DDPM.
\cite{Tachibana2021ItTaylorSS,JolicoeurMartineau2021GottaGF} utilized second-order numerical scheme to build a faster sampling process of DDPMs. 
\cite{song2020denoising} proposed DDIMs by using a non-Markovian diffusion processes to accelerate the sampling quality of DDPMs.
Later, DDIMs were further advanced by solving a differential equation on a manifold, so that the model can better approximate the ground-truth score functions~\cite{liu2021pseudo}.
\cite{luhman2021knowledge} distill a multi-step diffusion process into a single step via a student-teacher framework~\cite{fukuda2017efficient}.

Whilst speeding up the diffusion process, these existing methods are limited in the following aspects:
{\bf (1)} They focus on the simpler low-resolution generation tasks (\eg, CIFAR-10, CelebA$64\times64$) with less urgent acceleration needs.
It is unclear how scalable when they are applied to the more challenging high-resolution tasks (\eg, FFHQ$1024\times 1024$). 
In practice, the acceleration for low-resolution tasks is much less meaningful, since the speed of SGMs is already fairly good and much faster than the high-resolution case.
{\bf (2)}
All of them restrict to the conventional isotropic covariance matrix of the noises. In contrast, we show that using more general covariance matrix can enhance the performance of SGMs.
{\bf (3)}
Resort to existing SDE solvers or dynamics programming algorithms, they do not cast insightful theory for diffusion sampling acceleration. Instead, we provide a solid theoretical analysis of diffusion process in the frequency domain as well as a principled deviation.

\section{Preliminaries}\label{sec: pre}
\subsection{Score-based generative model (SGMs)}
Score matching~\cite{hyvarinen2005estimation} is a useful method for learning non-normalized statistical models. Given an unknown distribution $p_{\mathbf{x}^{\ast}}$ of vectors $\mathbf{x}^{\ast}\in \mathbb{R}^d$, it allows the model to directly estimate the \textit{score function} $\bigtriangledown_{\mathbf{x}}\log p_{\mathbf{x}^{\ast}}(\mathbf{x})$ (denoted as $\mathbf{s}_{\mathbf{x}^{\ast}}(\mathbf{x})$) at any point $\mathbf{x}$
with i.i.d. samples of $p_{\mathbf{x}^{\ast}}$.
SGMs aim to generate samples from $p_{\mathbf{x}^{\ast}}$ via score matching.
After score matching, SGMs generate a sample from $p_{\mathbf{x}^{\ast}}$ 
by simulating a Langevin dynamics at $T$ steps in a reverse order:
\begin{align}\label{eq: sample}
    \mathbf{x}_{t-1} = \mathbf{x}_t + \frac{\epsilon_t}{2}\mathbf{s}_{\mathbf{x}^{\ast}}(\mathbf{x}_t)+\sqrt{\epsilon_t}\mathbf{z}_t, \; t=T,\dots,1
\end{align}
where $\epsilon_t$ specifies the step size. The initial distribution $\mathbf{x}_T$ is sampled from a given prior distribution and the noise $\mathbf{z}_t$ are i.i.d. samples of the standard $d$-dimensional Gaussian distribution.
Under some regularity conditions, if $\epsilon_t\rightarrow 0,T\rightarrow +\infty$, we have the terminal distribution $p(\mathbf{x}_0)$ that approaches $p_{\mathbf{x}^{\ast}}$~\cite{welling2011bayesian}.
With this process, we transform a sample drawn from an initial Gaussian distribution to approach a desired distribution $p_{\mathbf{x}^{\ast}}$.
 
The first SGM is noise conditional score network (NCSN) \cite{song2019generative}. Denoting the NCSN model as $\mathbf{s}(\mathbf{x},\sigma)$. It is trained to match the score function $\bigtriangledown_{\mathbf{x}}\log p_{\sigma}(\mathbf{x})$, where $p_{\sigma}(\cdot)  = \mathcal{N}(\cdot;\mathbf{x}^{\ast},\sigma^2I_d)$ is a Gaussian distribution with $\sigma$ the variance of the additive noises. As $\sigma$ increases from $0$ to a large enough $\sigma_{L}$, $p_{\sigma}(\cdot)$ will transform gradually from the data distribution $p_{\mathbf{x}^{\ast}}$ to an approximately Gaussian distribution. Instead, the NCSN model $\mathbf{s}(\mathbf{x},\sigma)$ reverses the corruption process, starting from a Gaussian distribution and traveling to the target distribution $p_{\mathbf{x}^{\ast}}$ by an annealed Langevin dynamics.
\cite{song2020score} interpreted SGMs as discrete versions of the various stochastic differential equations (SDEs). 
In this perspective, diffusion sampling can be considered as a process in which we numerically solve a backward SDE. Based on this formula, \cite{song2020score} proposed the NCSN++ model good at high-resolution image generation.

\subsection{Limitation analysis}
Typically, the step size $\epsilon_t$ for sampling has to be small enough for ensuring stability in diffusion. Otherwise, the model will fail to converge to the target distribution~\cite{welling2011bayesian}, similar to the situation when applying a numerical solver for ODE (ordinary differential equation) and SDE.
As a result, existing SGMs suffer from 
{\bf \em a lengthy sampling process} as many iterations are needed.
This is inefficient and unscalable as compared to GANs, especially for the high-resolution image generation.
We conjugate that the inconvergence issue with acceleration is mainly caused by the noise term $\sqrt{\epsilon_t}\mathbf{z}_t$, which is under-controlled when $\epsilon_t$ is too large. 
At every step $t$, the SGM needs to push the sample $\mathbf{x}_t$ closer to the target distribution. The noise term, however, has the risk to counter this effect by 
pulling the sample further from the target distribution.
Intuitionally, suppose we add a noise closer to the target distribution, this drifting risk would be alleviated since it is easier for the SGM to take the control of the sampling process. In this work, we show that the convergence of the sampling process can be sped up dramatically when we constraint the noise term closer to the target distribution both in the spatial and frequency domain.

\section{Method}

\subsection{Motivation}
We consider that the above low-efficiency limitation with existing SGMs is caused by the ignorance of the target data distribution in model inference. Specifically, for both the initial point $\mathbf{x}_T$ and noises $\mathbf{z}_t$, existing models typically use \textbf{\em isotropic Gaussian distributions} (with covariance matrix $\sigma^2I_d$) in inference.
This is consistent with the training's setup arguably considered to be the best option.

However, as we prove in Theorem~\ref{thm: deviation}, if the initial point and the noises added in the diffusion process distribute similarly to the target distribution, the model inference can be favorably accelerated with better stability, {\em without model retraining}.
Critically, it is shown that the distributional shift in sample initialization and noises
w.r.t. the training setup is not harmful at all to the generative quality.
Our theory is in stark contrast to and challenges
the conventional wisdom.

Given the same sampling process as Eq.~\eqref{eq: sample}, we aim to redesign the noise $\mathbf{z}_t$ for model inference in a way that it can still generate samples at high quality under acceleration.
Algorithmically, in contrast to existing alternatives, we do not limit the covariance matrix (\ie, the variance schedule) of the noise to the form of $\sigma^2I$. We leverage the prior knowledge of the target distribution to design more flexible and relevant noises, where {\em the covariance matrix can be any positive definite matrix.
}

\subsection{Diffusion in the frequency domain}
We first show that the sampling process can be equivalently transformed to the frequency domain (Theorem~\ref{Thm: transform}). This facilitates us to explore the prior knowledge of the target task in the frequency domain.
We start with the following lemma.

\begin{lemma}\label{lem transform}
If two $d$-dimensional random vectors $\mathbf{x},\mathbf{y}\in\mathbb{R}^d$ have differentiable density functions, and satisfy $\mathbf{y}=G\mathbf{x}$, where the matrix $G\in\mathbb{R}^{d\times d}$ is invertible, we have
\begin{align*}
    \bigtriangledown_{\mathbf{y}}\log p_{\mathbf{y}}(\mathbf{y})  = \bigtriangledown_{\mathbf{y}}\log p_{\mathbf{x}}(G^{-1}\mathbf{y})=\bigtriangledown_{\mathbf{y}}\log p_{\mathbf{x}}(\mathbf{x}).
\end{align*}
\end{lemma}
\begin{proof}
Note for any invertible differentiable transformation $g\in \mathbb{R}^d \rightarrow \mathbb{R}^d$, if $\mathbf{y}=g(\mathbf{x})$, we have
\begin{align*}
    p_{\mathbf{y}}(\mathbf{y}) = p_{\mathbf{x}}(g^{-1}(\mathbf{y}))\left |\det\left[ \frac{d g^{-1}(\mathbf{y})}{d\mathbf{y}}\right] \right |.
\end{align*}
In particular, $p_{\mathbf{y}}(\mathbf{y}) = p_{\mathbf{x}}(G^{-1}\mathbf{y})\left | G \right |^{-1}$. We verify the lemma by taking the logarithm and calculating the gradients at both sides of the equation.
\end{proof}

\begin{theorem}\label{Thm: transform}
Suppose $F\in \mathbb{R}^{d\times d}$ is an orthogonal matrix, then the diffusion process (Eq. \eqref{eq: sample})
can be rewritten as
\begin{align}\label{eq:freq_diffusion}
\tilde{\mathbf{x}}_{t-1}  = \tilde{\mathbf{x}}_t + \frac{\epsilon_t}{2}\bigtriangledown_{\tilde{\mathbf{x}}}\log p_{\tilde{\mathbf{x}}^{\ast}}(\tilde{\mathbf{x}}_t)+\sqrt{\epsilon_t}F\mathbf{z}_t,
\end{align}
where $\tilde{\mathbf{x}} = F\mathbf{x}$, and $\{\mathbf{z}_t\}$ do not need to follow isotropic Gaussian distributions.
\end{theorem}

\begin{proof}
Multiple $F$ at both sides of the the original diffusion process, we have
\begin{align*}
    \tilde{\mathbf{x}}_{t-1}  = \tilde{\mathbf{x}}_t + \frac{\epsilon_t}{2}F\bigtriangledown_{\mathbf{x}}\log p_{\mathbf{x}^{\ast}}(\mathbf{x}_t)+\sqrt{\epsilon_t}F\mathbf{z}_t.
\end{align*}
By Lemma~\ref{lem transform}, we have $\bigtriangledown_{\tilde{\mathbf{x}}}\log p_{\tilde{\mathbf{x}}^{\ast}}(\tilde{\mathbf{x}}) = \bigtriangledown_{\tilde{\mathbf{x}}}\log p_{\mathbf{x}^{\ast}}(\mathbf{x})$. We also have $\bigtriangledown_{\mathbf{x}} = F^{\mathsf{T}}\bigtriangledown_{\tilde{\mathbf{x}}}$ by the chain rule. Putting all the things together, we have
\begin{align*}
    \tilde{\mathbf{x}}_{t-1}  = \tilde{\mathbf{x}}_t + \frac{\epsilon_t}{2}FF^{\mathsf{T}}\bigtriangledown_{\tilde{\mathbf{x}}}\log p_{\tilde{\mathbf{x}}^{\ast}}(\tilde{\mathbf{x}}_t)+\sqrt{\epsilon_t}F\mathbf{z}_t.
\end{align*}
As $F$ is orthogonal, $FF^{\mathsf{T}} = I_d$. We thus finish the proof.
\end{proof}

\begin{remark}
According to this theorem, we can directly transform a diffusion process by an orthogonal transformation, with the only difference that the noise is altered to $F\mathbf{z}_t$. Particularly, if $\mathbf{z}_t$ obeys an isotropic Gaussian distribution applied by existing methods,
$F\mathbf{z}_t$ obeys the exact the same distribution as $\mathbf{z}_t$.
Hence, \textbf{the existing diffusion sampling methods are orthogonally invariant.} 
\end{remark}

As {\em any permutation matrix} is orthogonal, a diffusion process using the noises in istropic Gaussian distributions \textbf{\em ignores the coordinate order in data}. This property might be reasonable for {\em unstructured} data, but not for {\em structured data} typical in various applications such as natural images where the coordinate order represents the location information of individual pixels locally and the semantic structures globally.
Hence, overlooking the coordinate order in diffusion could be detrimental.

More specifically, consider the popular image generation task. Denote an image dataset $\mathcal{D} = \{\mathbf{x}^{(i)}\in\mathbb{R}^{C\times H\times W},i=1,\dots,N\}$, where $C,H$ and $W$ are the channel number, height and width of images, respectively. Substitute $F$ in Eq.~\eqref{eq:freq_diffusion} by the the two-dimensional type II discrete cosine transform (DCT) that is denoted as $D[\cdot]$ (see Sec.~\ref{sec:dct} in the Appendix for a full definition), we have the diffusion process in the frequency domain as
\begin{align}\label{eq:dct_diffusion}
    \tilde{\mathbf{x}}_{t-1}  = \tilde{\mathbf{x}}_t + \frac{\epsilon_t}{2}\bigtriangledown_{\tilde{\mathbf{x}}}\log p_{\tilde{\mathbf{x}}^{\ast}}(\tilde{\mathbf{x}}_t)+\sqrt{\epsilon_t}D[\mathbf{z}_t].
\end{align}
It is a fact that the amplitude of low-frequency signals of natural images is typically much higher than that of high-frequency ones~\cite{Schaaf1996ModellingTP}. That is, the coordinates of high-frequency part of the target distribution $p_{\tilde{\mathbf{x}}^{\ast}}$ have much less scales than that of the low-frequency part. Differently, the noise term $\sqrt{\epsilon_t}D[\mathbf{z}_t]$ added by the existing SGMs exhibit the same scales over the coordinates in the frequency domain.
This is because obeying the isotropic Gaussian,
$\sqrt{\epsilon_t}\mathbf{z}_t$ and $\sqrt{\epsilon_t}D[\mathbf{z}_t]$ share the same distribution.
In the following, we will address this issue by reducing this coordinate-wise gap between $p_{\tilde{\mathbf{x}}^{\ast}}$ and $\sqrt{\epsilon_t}D[\mathbf{z}_t]$ using the prior knowledge from the target distribution.

\subsection{Target distribution aware sampling}
\label{sec: filter}

We introduce the idea of {\bf \em Target Distribution Aware Sampling} (TDAS) for efficient SGMs.
Specifically, we make sample initialization and noise sampling be aware of the target data distribution so that the diffusion process can be accelerated during model inference.

With the above theoretical findings, as an instantiation we formulate 
a \textbf{\em space-frequency filter} operation to implement TDAS for image generation.
It regulates the distributions of sample initialization 
and additive noises in both space and frequency domains concurrently.
Formally, we regulate the initial sample $\mathbf{x}_T$ as
\begin{equation}
    \bm{\eta}_T =  D^{-1}[M_{\text{freq}}\odot D\big[M_{\text{space}}\odot \mathbf{x}_T]\big],
\end{equation}
and each additive noise $\mathbf{z}_t$ as 
\begin{equation}
    \bm{\eta}_t = D^{-1}\big[M_{\text{freq}}\odot D[M_{\text{space}}\odot \mathbf{z}_t]\big],
\end{equation}
where $\mathbf{x}_T,\mathbf{z}_t\sim \mathcal{N}(0,I_{C\times H\times W})$ and $\odot$ denotes element-wise product. 
The linear operator $D^{-1}$ is the inverse of two-dimensional DCT $D$.
Critically, the space $M_{\text{space}}$ and frequency $M_{\text{freq}}$ filters, the key design of TDAS, are {\em target distribution aware} with the same shape as $\mathbf{x}_T$.
Their functions here are to regulate the coordinates of initial sample and noises in their respective domain,
subject to the prior structural information of image distribution. 
Note that now the covariance matrix of $\bm{\eta}_t$ can be any general positive definite matrix.
Seemingly, TDAS had a risk of changing the convergence distribution of the diffusion process, according to an analysis of its Fokker-Plank equation~\cite{gardiner1985handbook}.
This analysis is only verified by assuming that (1) the step size is small enough and (2) the SGM learns the exact score function. However, both are invalid in our context: the step size is large under acceleration, and the SGMs often only approximate the ground-truth score function near a sub-manifold~\cite{liu2021pseudo}.
Hence we consider the deviation from the target distribution as analyzed above dominates and
forms a major obstacle for fast convergence (Theorem~\eqref{thm: deviation}).

{\bf Frequency filter. }
To leverage the statistical frequency prior of images that 
the low frequency signals have much larger amplitude 
than the high frequency counterpart~\cite{Schaaf1996ModellingTP},
we design $M_{\text{freq}}$ as:
\begin{align}\label{eq: freq_mask}
M_{\text{freq}}(c,h,w) = 
\left\{\begin{matrix}
 1&,& h^2+w^2 \leq 2 r_{\text{th}}^2\\
\lambda&,&\text{otherwise}
\end{matrix}\right.,1\leq c \leq C,1\leq h \leq H,1\leq w \leq W,
\end{align}
where $r_{\text{th}}$ is the threshold radial and $\lambda \in (0,1)$ specifies the rate at which we suppress the high frequency zone. In this way, initial sample and noises $\{\bm{\eta}_t\}_{t=1}^{T}$ all exhibit a similar frequency distribution as general images,
\ie, target distribution awareness. 
Besides, considering that the high-frequency details are most challenging to learn~\cite{xu2019training,xu2019freq},
suppressing their amplitude implies some reduction in the generative task difficulty.
For more fine-grained regulation, we can design multiple zones each with a dedicated suppressing rate. 
Without expensive grid search, we find it is empirically easy to tune the parameters based on the statistics of the natural images and the samples generated by the vanilla SGMs in the frequency domain (see Sec.~\ref{sec:freq_analysis} in Appendix for more details).

{\bf Space filter.}
For space filter, we consider two situations.
When the SGM (\eg, NCSN) is limited in learning spatial structure information, or the target image data exhibit consistent structure priors
(\eg, facial images),
we exploit the {\em average} statistics.
Specifically, we design the space filter operator with an image set $\mathcal{D}$ as: 
\begin{align*}
    M_{\text{space}}(c,h,w) = \rho\Big(\frac{1}{\left |\mathcal{D} \right |}\sum_{\mathbf{x}^{(i)} \in \mathcal{D}} \left | \mathbf{x}^{(i)}(c,h,w) \right | \Big),1\leq c \leq C,1\leq h \leq H,1\leq w \leq W,
\end{align*}
where $\rho(\cdot)\in\mathbb{R}\rightarrow\mathbb{R}$ is a non-negative 
monotonically increasing function. We choose $\rho(x) = \log(1+x)$ so that the scales do not vary dramatically across different coordinates. In practice, to make it easier to control the step size we also normalize and soften $M_{\text{space}}$ for stability; see Sec.~\ref{sec:freq_analysis} in Appendix for more detail.
In otherwise cases, we simply set all the elements of $M_{\text{space}}$ to 1.

It is worthwhile noting that the original sampling method
is a special case of our formulation, when setting both operators to all-1 matrix.
We summarize our method in Alg.~\ref{algo: sample}.

\begin{algorithm}[tb]
  \caption{Target distribution aware sampling}
  \label{algo: sample}
\begin{algorithmic}
  \State {\bfseries Input:} 
  The space
  $M_{\text{space}}$ and frequency $M_{\text{freq}}$ filters,
  the sampling iterations $T$; 
  \State {\bfseries Diffusion process:}
  \State Drawing an initial sample $\mathbf{z}_T \sim \mathcal{N}(0,I_{C\times H \times W})$
  \State {\em Applying TDAS}: $\mathbf{x}_T \leftarrow D^{-1}[M_{\text{freq}}\odot D[M_{\text{space}}\odot \mathbf{z}_T]]$
  \For{$t=T$ {\bfseries to} $1$}
      \State Drawing a noise $\mathbf{z}_t\sim \mathcal{N}(0,I_{C\times H \times W})$
      \State {\em Applying TDAS}: $\bm{\eta}_t \leftarrow D^{-1}[M_{\text{freq}}\odot D[M_{\text{space}}\odot \mathbf{z}_t]]$
        \State Diffusion $\mathbf{x}_{t-1} \leftarrow \mathbf{x}_t + \frac{\epsilon_t}{2}\mathbf{s}_{\mathbf{x}^\ast}(\mathbf{x}_t)+\sqrt{\epsilon_t}\bm{\eta}_t$
  \EndFor
  \State {\bfseries Output:} $\mathbf{x}_0$
\end{algorithmic}
\end{algorithm}

\begin{remark}
For the computational complexity of TDAS, the major overhead is from DCT and its inverse. Their computational complexity is $O(d\log d)$ by using Fast Fourier Transform (FFT). For an image $\mathbf{x}\in\mathbb{R}^{C\times H\times W}$, the complexity per iteration in the sampling process is $O(CHW(\log H+\log W) )$, which is marginal compared to the whole diffusion complexity.
\end{remark}

\subsection{Deviation analysis}\label{sec: deviation}
We provide a deviation analysis of the diffusion process (Eq.~\eqref{eq: sample}) to clarify the effect of TDAS on convergence.
\begin{theorem}\label{thm: deviation}
Suppose the diffusion process (Eq.~\eqref{eq: sample}) converges to a distribution $\mathbf{x}^{\ast} \sim p_{\mathbf{x}^{\ast}}$. Denoting $\mathcal{F}_{-t}$ as the $\sigma$-algebra generated by $\{\mathbf{x}_T,\mathbf{z}_s,s=T,\dots,t+1\}$. If the additive noises $\{\mathbf{z}_s\}^{T}_{s=t+1}$ depend on $\mathbf{x}_t$ and satisfy the condition $\mathbb{E}\left [\mathbf{z}_t \mid \mathcal{F}_{-t}\right] = 0,\forall t \in \{T,\dots,1\}$, then the deviation of $\mathbf{x}_t$ from $\mathbf{x}^{\ast}$
can be written as
\begin{align*}
\mathbb{E}\left [ \left \|\mathbf{x}^{\ast} -\mathbf{x}_{t-1} \right \|^2 \right]=C_1+\epsilon_t\mathbb{E}\left [ \left \|\mathbf{z}_t \right \|^2 \right]
       -2\sqrt{\epsilon_t}\mathbb{E}\left [ \mathbf{x}^{\ast}\cdot \mathbf{z}_t \right],
\end{align*}
where the term $C_1$ is a constant independent of $\mathbf{z}_t$. 
\end{theorem}
\begin{proof}
For the conditional deviation, we have
\begin{align*}
        &\mathbb{E}\left [ \left \|\mathbf{x}^{\ast} -\mathbf{x}_{t-1} \right \|^2 \bigg | \mathcal{F}_{-t} \right]
     = \mathbb{E}\left [ \left \|\mathbf{x}^{\ast} -\mathbf{x}_t - \frac{\epsilon_t}{2}\mathbf{s}_{\mathbf{x}^{\ast}}(\mathbf{x}_t)-\sqrt{\epsilon_t}\mathbf{z}_t \right \|^2 \bigg | \mathcal{F}_{-t} \right]\\
      &=\mathbb{E}\left [ \left \|\mathbf{x}^{\ast} -\mathbf{x}_t  - \frac{\epsilon_t}{2}\mathbf{s}_{\mathbf{x}^{\ast}}(\mathbf{x}_t) \right \|^2 \bigg | \mathcal{F}_{-t} \right]
      +\epsilon_t\mathbb{E}\left [ \left \|\mathbf{z}_t    \right \|^2 \bigg | \mathcal{F}_{-t} \right]-2\sqrt{\epsilon_t}\mathbb{E}\left [ \mathbf{x}^{\ast}\cdot \mathbf{z}_t  \bigg | \mathcal{F}_{-t} \right]
\end{align*}
The last equation is due to
\begin{align*}
    \mathbb{E}\left [ \big(\mathbf{x}_t+ \frac{\epsilon_t}{2}\mathbf{s}_{\mathbf{x}^{\ast}}(\mathbf{x}_t)\big)\cdot \mathbf{z}_t \bigg | \mathcal{F}_{-t} \right]
    =\big(\mathbf{x}_t+ \frac{\epsilon_t}{2}\mathbf{s}_{\mathbf{x}^{\ast}}(\mathbf{x}_t)\big)\cdot \mathbb{E}\left [ \mathbf{z}_t \bigg | \mathcal{F}_{-t} \right]=0.
\end{align*}
Then we take the expectation of both sides
\begin{align*}
    \mathbb{E}\left [ \left \|\mathbf{x}^{\ast} -\mathbf{x}_{t-1} \right \|^2 \right]
     =\mathbb{E}\left [ \left \|\mathbf{x}^{\ast}-\mathbf{x}_t  - \frac{\epsilon_t}{2}\mathbf{s}_{\mathbf{x}^{\ast}}(\mathbf{x}_t) \right \|^2 \right] +\epsilon_t\mathbb{E}\left [ \left \|\mathbf{z}_t \right \|^2 \right]-2\sqrt{\epsilon_t}\mathbb{E}\left [ \mathbf{x}^{\ast}\cdot \mathbf{z}_t \right],
\end{align*}
and the theorem is proved.
\end{proof}
\begin{remark}
If the condition $\mathbb{E}\left [\mathbf{z}_t \mid \mathcal{F}_{-t}\right] = 0$ is replaced by a weaker one $\mathbb{E}\left [\mathbf{z}_t \right] = 0$, the claim is not true.
The condition $\mathbb{E}\left [\mathbf{z}_t \mid \mathcal{F}_{-t}\right]$ means that at the step $t$, given $\{\mathbf{x}_T,\dots,\mathbf{x}_{t+1}\}$, the conditional expectation of $\mathbf{z}_t$ is $0$. Given that the diffusion process of all existing SGMs satisfies this setting, our theorem applies generally.
\end{remark}
Suppose we keep the noise variance $\mathbb{E}\left [ \left \|\mathbf{z}_t \right \|^2 \right]$ unchanged, $\mathbf{z}_t$ impacts the deviation only through the correlation term $-2\sqrt{\epsilon_t}\mathbb{E}\left [ \mathbf{x}^{\ast}\cdot \mathbf{z}_t \right]$. 
This suggests that if we have the elements of $\mathbf{z}_t$ positively correlate with the corresponding elements of $\mathbf{x}^{\ast}$, this deviation will decrease, encouraging the diffusion convergence. As normally $\epsilon \ll 1$, the correlation term is more influential than the variance term $\mathbb{E}\left [ \left \|\mathbf{z}_t \right \|^2 \right]$. Therefore, if we replace the noise $\mathbf{z}_t$ with one that has higher similarity to the target distribution, the convergence will become both more stable and faster. 
Similarly, let $t=T$, higher distributional similarity of the starting point $\mathbf{x}_T$ with $\mathbf{x}^{\ast}$ also promotes convergence. 
\section{Experiments}

\subsection{Experimental setup}

\textbf{Datasets.}
We use 
CIFAR-10~\cite{krizhevsky2009learning}, CelebA~\cite{liu2015faceattributes} LSUN (bedroom and church)~\cite{yu2015lsun}, and FFHQ~\cite{karras2019style} datasets. 
We adopt the same preprocessing as~\cite{song2019generative,song2020score}.

\textbf{SGMs.}
We evaluate TDAS on two representative SGMs including NCSN~\cite{song2019generative} and NCSN++~\cite{song2020score}~(Sec.~\ref{sec: pre}). 
For dedicated evaluation, we focus on comparing our TDAS with the original ({\bf \em vanilla}) sampling methods in their optimal setup, \eg, the suggested step size.

\textbf{Implementation.}
We conduct all experiments with PyTorch, on NVIDIA RTX 3090 GPUs.
We use the public codebase of
NCSN\footnote{\url{ https://github.com/ermongroup/ncsn}}
and NCSN++\footnote{\url{https://github.com/yang-song/score\_sde}}.
We use the released checkpoints of NCSN and NCSC++ for all datasets (except CelebA for which there is no released checkpoint, and we train by ourselves with the released codes instead).
For fair comparison, we only replace the vanilla sampling method with our TDAS whilst keeping the remaining unchanged during inference for a variety of sampling iterations
For TDAS, we apply the average space filter for all three diffusion baselines on FFHQ and NSCN on CIFAR-10,
whilst the identity space filter for the remaining.
Whenever reducing the iterations $T$, we expand the step size proportionally
for consistent accumulative update.

\subsection{Evaluation on high resolution image generation}
{\bf LSUN. }
Using NCSN++ trained on LSUN (church and bedroom) ($256\times256$), 
it is shown in Fig.~\ref{fig:bedroom_church}: 
(1) 
When reducing the sampling iterations from the default 2000 to 400, the vanilla sampling method degrades the generation quality significantly.
(2) On the contrary, our TDAS can mitigate this problem well
based on our target distribution aware sampling idea.

{\bf FFHQ. }
For more challenging high resolution image generation, we evaluate TDAS on FFHQ dataset ($1024\times 1024$). 
We use the state-of-the-art diffusion method NCSN++.
We reduce the default sampling iterations 2000 down to 90.
We draw several observations from Fig.~\ref{fig: demo_face}:
(1) With thousands of sampling iterations, both sampling methods
yield highly realistic facial images. However, this is computationally expensive.
(2) When the sampling iterations reduced to 125, the vanilla method 
already fails to produce high-frequency details. Further reduction to 100
causes a complete failure.
(3) On the contrary, under the same sampling reduction our TDAS can still produce facial images with marginal quality degradation.
This suggests {\em an even bigger advantage of our method comparing to low-resolution image generation.}

\textbf{Quantification. }
We compare the clean-FID scores~\cite{parmaraliased} of FFHQ ($1024\times1024$) samples generated by NCSN++~\cite{song2020score} without and with our TDAS. For each case, we generate 50,000 samples under 100 iterations.
Table~\ref{tab:ffhq_fid} shows that TDAS is highly effective in improving the quality of high-resolution images generated under acceleration.
To the best of our knowledge, this is the first reported FID evaluation of SGMs on the challenging FFHQ ($1024\times1024$) image generation task.

\begin{table}[h]
\vspace{-1ex}
\begin{center}
\caption{FID scores of NCSN++ with and without TDAS on generating facial images. Dataset: FFHQ. Resolution: $1024\times1024$. Iterations: 100.}\label{tab:ffhq_fid}
   \begin{tabular}{ccc}
\toprule[2pt]
    &  w/o TDAS & w/ TDAS \\ \midrule
FID & 442.82                  & \textbf{38.68}       \\ \bottomrule[2pt]
\end{tabular} 
\end{center}
\vspace{-2ex}
\end{table}

\textbf{Running speed}
We compare the running speed between the vanilla and our TDAS in NCSN++ on one NVIDIA RTX 3090 GPU.
As shown in Table \ref{tab: time}, our TDAS can significantly reduce the running cost, particularly for high resolution image generation on FFHQ dataset.

\begin{table}[ht]
\vspace{-3ex}
\centering
\caption{Comparing the running time of generating a batch of $8$ images  with various resolutions using NCSN++~\cite{song2020score}. We set the iterations of TDAS so that it can generate images of equal or better quality than the vanilla NCSN++. The number in the bracket means the iteration number.
{\em Time unit:} Seconds.}
\label{tab: time}
\setlength{\tabcolsep}{4mm}{
\begin{tabular}{llll}
\toprule[2pt]
              \multicolumn{1}{c}{Dataset} & \multicolumn{1}{c}{CelebA} & \multicolumn{1}{c}{LSUN}  & \multicolumn{1}{c}{FFHQ} \\ 
             \multicolumn{1}{c}{Resolution} & \multicolumn{1}{c}{$64\times64$} & \multicolumn{1}{c}{$256\times256$}  & \multicolumn{1}{c}{$1024\times1024$} \\ \midrule
\multicolumn{1}{c}{Vanilla}     & \multicolumn{1}{c}{33.5 (333)}  & \multicolumn{1}{c}{1021.2 (2000)} & \multicolumn{1}{c}{1758.6 (2000)}     \\
\multicolumn{1}{c}{\bf TDAS}    & \multicolumn{1}{c}{\textbf{7.9} (66)}      & \multicolumn{1}{c}{\textbf{208.4} (400)}  & \multicolumn{1}{c}{\textbf{95.8} (100)}      \\ \midrule
\multicolumn{1}{c}{\em Speedup times} & \multicolumn{1}{c}{4.2} & \multicolumn{1}{c}{4.9}   & \multicolumn{1}{c}{18.4}      \\ \bottomrule[2pt]
\end{tabular}}
\end{table}

\begin{figure}[h]
		\hspace{-2ex}
		\centerline{\includegraphics[scale=0.155]{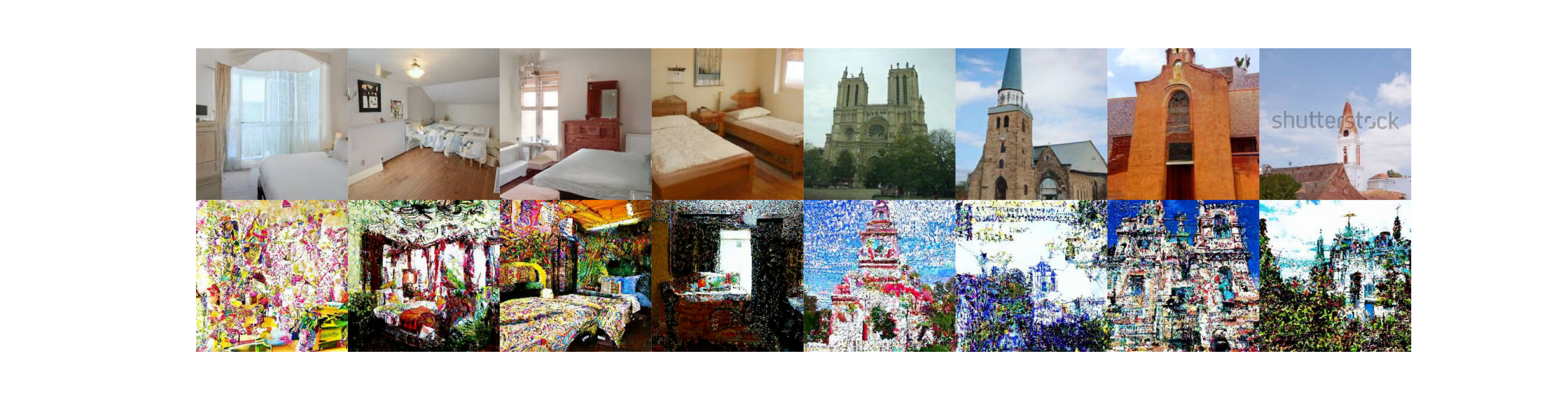}}
		\vspace{-3ex}
		\caption{LSUN (bedroom and church) images generated by TDAS and and vanilla sampling (bottom) under $400$ sampling iterations using NCSN++~\cite{song2020score}. More examples in Appendix.}\label{fig:bedroom_church}
\end{figure}

\textbf{Ablation studies}
We investigate the effect of adjusting the initial distribution and noise distribution separately.
We use NCSN++, with the sampling iterations $133$. Fig.~\ref{Fig: ablation2_half} (a) demonstrates samples produced by TDAS, TDAS without adjusting initial distribution, TDAS without adjusting noise distribution and original sampling method on FFHQ. 
It is observed that adjusting the noise distribution plays the main role in convergence acceleration.

We also investigate the effect of both $M_{\text{space}}$ and $M_{\text{freq}}$ separately. 
We use NCSN++, with the sampling iterations $133$.
Fig.~\ref{Fig: ablation2_half} (b) shows the samples produced by TDAS on FFHQ, where we remove $M_{\text{space}}$, or $M_{\text{freq}}$ or both (\ie, the vanilla sampling method). It is observed that both filters can enhance the quality of image synthesis, and the space one
plays a bigger role in this case.

\begin{figure}[ht]
		\begin{subfigure}{0.48\linewidth}
		\includegraphics[scale=0.226]{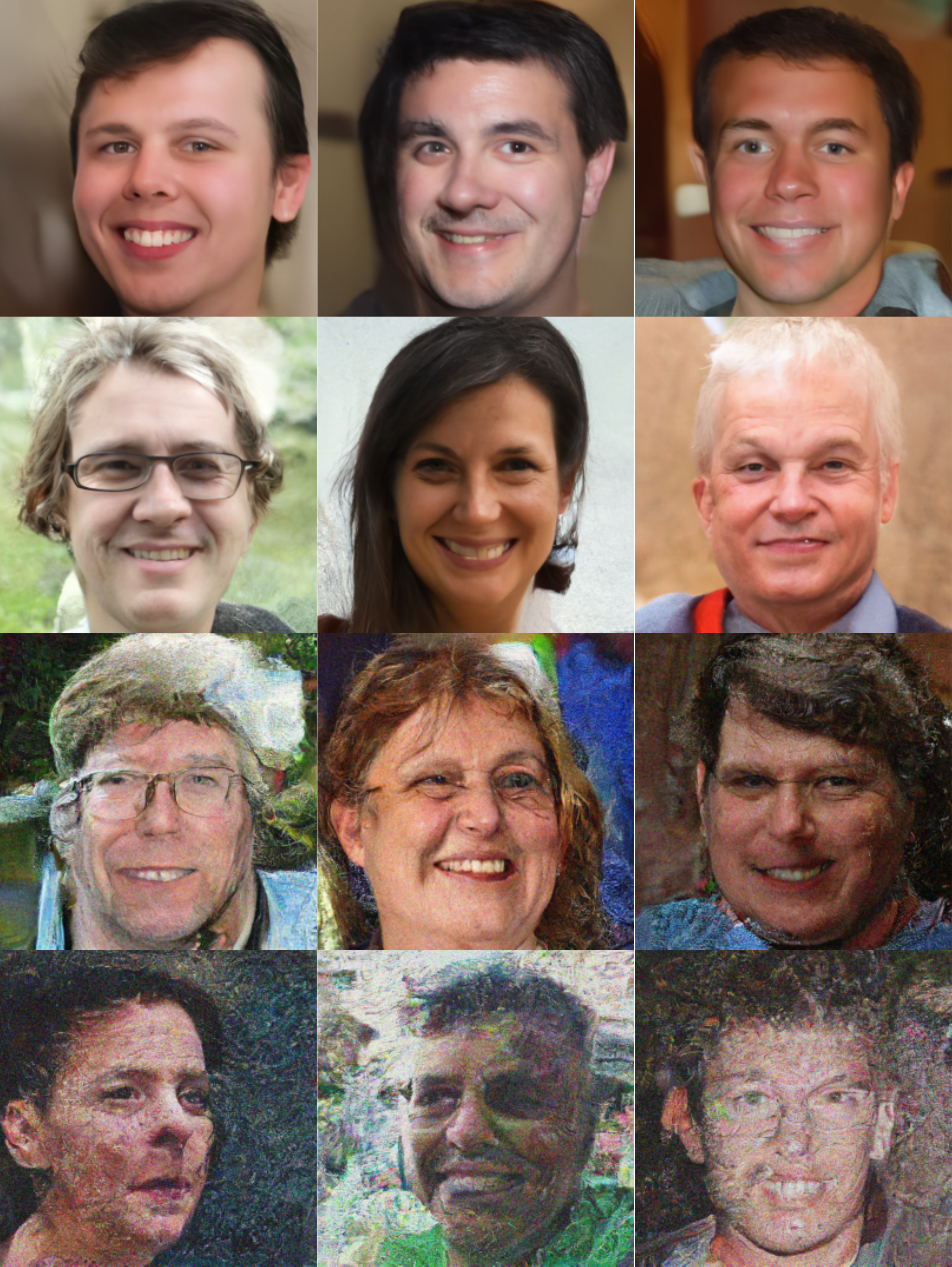}
		\caption{}
		\end{subfigure}
		\hspace{2em}
		\begin{subfigure}{0.48\linewidth}
        \includegraphics[scale=0.226]{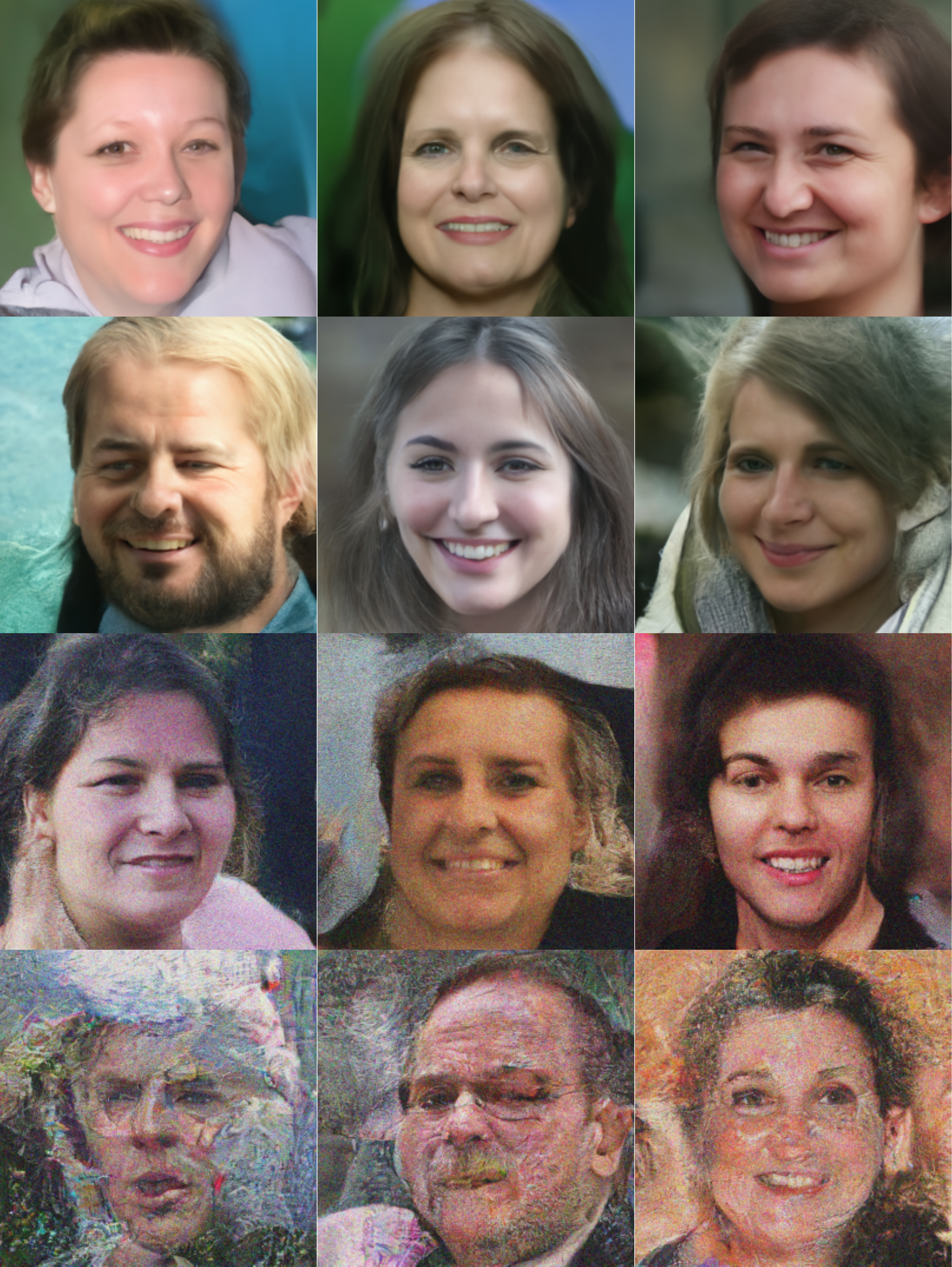}
        \caption{}
		\end{subfigure}
		\caption{Facial images produced by different versions of TDAS and the vanilla sampling on FFHQ ($1024\times 1024$). We use NCSN++~\cite{song2020score} with $T = 133$ sampling iterations. (a) From the top to bottom: TDAS, TDAS without regulating the initial distribution, TDAS without regulating the noise distribution, and the vanilla sampling. (b) From the top to bottom: TDAS, TDAS without frequency filter, TDAS without space filter, and the vanilla sampling.
		}\label{Fig: ablation2_half}
\end{figure}
\vspace{-1ex}

\subsection{Evaluation on low resolution image generation}

{\bf CelebA. }
We conduct both sample demonstration and the FID~\cite{heusel2017gans} scores on CelebA ($64\times64$) as shown in Fig~\ref{fig:celeba}. We apply both space and frequency filters for NCSN++.
It is observed that TDAS can dramatically reduce the FID compared to the vanilla NCSN++.

{\bf CIFAR-10. }
We conduct an FID~\cite{heusel2017gans} score based quantitative evaluation on CIFAR-10. 
We apply both space and frequency filters for NCSN,
and NCSN++.
It is observed in Table~\ref{tab:ncsn_cifar10} that TDAS helps with reducing FID scores across all sampling iterations.
Note, given low-resolution images in CIFAR-10, the benefit of space-frequency filter is relatively less than the cases with higher resolutions (\ie, LSUN, CelebA and FFHQ).
\begin{table}[h]
\vspace{-2ex}
\caption{FID scores of the SGMs with and without TDAS under different iterations on CIFAR-10 generated by NCSN (a) and NCSN++ (b).}\label{tab:ncsn_cifar10}
\begin{subtable}[t]{0.49\linewidth}
\caption{NCSN~\cite{song2019generative}}
\setlength{\tabcolsep}{4mm}{
\begin{tabular}{ccc}
\toprule[2pt]
\textbf{Iterations} & w/o TDAS & w/ TDAS \\ \midrule
1000       & 25.07   & \textbf{23.56}          \\
500       & 38.53    & \textbf{27.30}          \\
333       & 53.25    & \textbf{40.71}          \\
250       & 67.95    & \textbf{56.43}          \\
200       & 83.14    & \textbf{72.92}          \\
\bottomrule[2pt]
\end{tabular}}
\end{subtable}
\hspace{1em}
\begin{subtable}[t]{0.49\linewidth}
\caption{NCSN++~\cite{song2020score}}
\setlength{\tabcolsep}{4mm}{
\begin{tabular}{ccc}
\toprule[2pt]
\textbf{Iterations} & w/o TDAS & w/ TDAS \\ \midrule
200       & 4.35     & \textbf{2.97}          \\
166       & 6.24     & \textbf{2.78}          \\
125       & 12.70    & \textbf{3.10}          \\
111       & 19.01    & \textbf{5.67}          \\
100       & 29.39    & \textbf{7.78}          \\ \bottomrule[2pt] 
\end{tabular}}
\end{subtable}

\end{table}

\vspace{-4ex}
\begin{figure}[h]
        \hspace{-14ex}
        \includegraphics[scale=0.33]{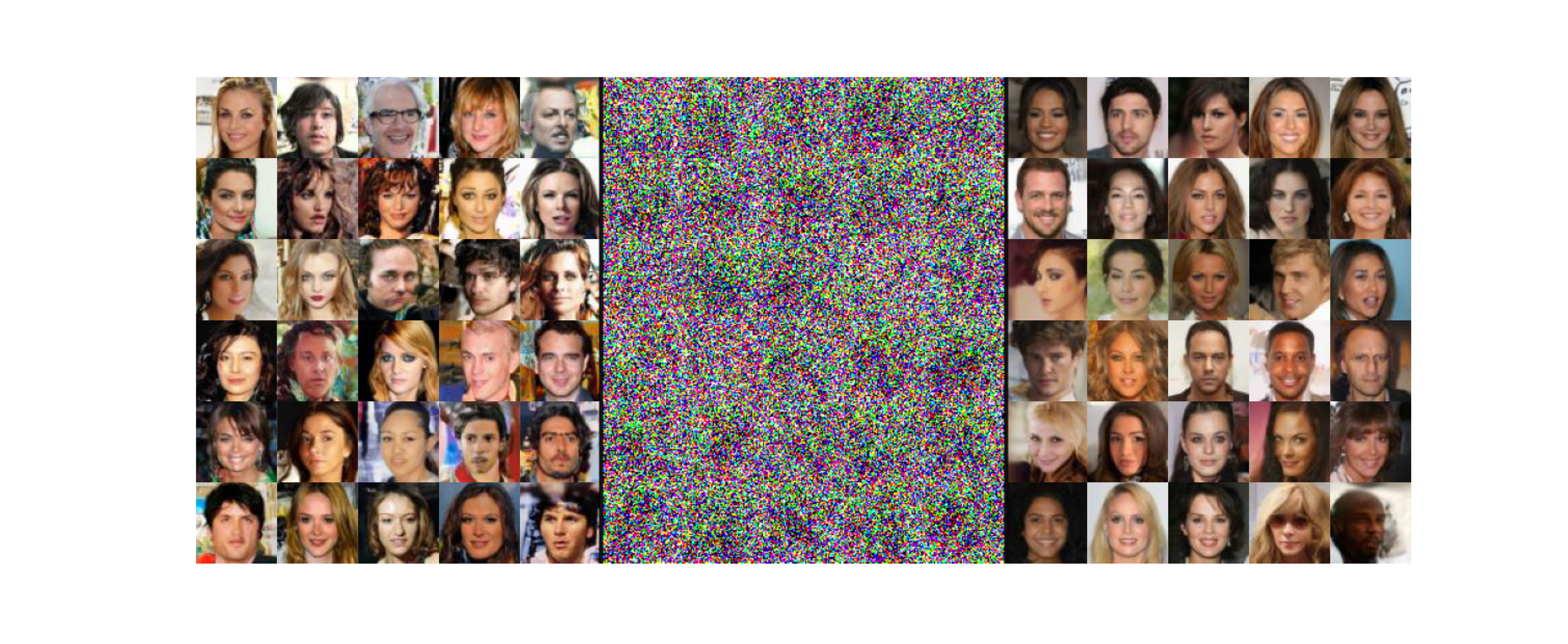}
        
        \vspace{-5ex}
                \hspace{9ex}$\text{FID}=14.5$\hspace{17ex}$\text{FID}=436.9$\hspace{17ex}$\text{FID}=9.1$
    
        \caption{
        Facial images generated on CelebA~\cite{yu2015lsun}.
        \textbf{\em Left}: NCSN++~\cite{song2020score} without TDAS using 333 sampling iterations.
        \textbf{\em Middle}: NCSN++~\cite{song2020score} without TDAS using 66 sampling iterations.
        \textbf{\em Right}: NCSN++ with TDAS using 66 sampling iterations.
        {\em Resolution}: $64 \times 64$.
        }
        \label{fig:celeba}
\end{figure}

\section{Conclusion}
We propose a novel Target Distribution Aware Sampling (TDAS) approach to accelerating existing score-based generative models (SGMs). 
Theoretically, we prove that having higher similarity between the diffusion related distributions and the target distribution encourages sampling convergence in diffusion. 
Further, we formulate a theoretical diffusion process in the frequency domain with an instantiation for image generation with visual structural priors. 
Experimentally, we show that TDAS can significantly accelerate off-the-shelf SGMs while maintaining most generation quality,
especially on high-resolution generation tasks.

\bibliographystyle{plain}
\bibliography{reference}

\appendix

\newpage

\section{Appendix}

\subsection{Discrete cosine transform and discrete Fourier transform}\label{sec:dct}
\textbf{Discrete cosine transform (DCT). }
Let us start with the one-dimensional type II discrete cosine transform (DCT)~\cite{ahmed1974discrete} formulated as
\begin{align*}
    D_1[\mathbf{x}](k)  = \sqrt{\frac{2}{d}}(\frac{1}{\sqrt{2}}\cos \big(\frac{\pi}{d}\frac{k}{2}\big)\mathbf{x}(0)+
    \sum_{n=1}^{d-1}\cos  (\frac{\pi}{d}(n+\frac{1}{2})k)\mathbf{x}(n)),
    k=0,\dots,d-1,
\end{align*}
where $\mathbf{x}(n)$ is the $n$-th coordinate of a $d$-dimensional vector $\mathbf{x}$.

Similarly, for image data with a sample $\mathbf{x} \in \mathbb{R}^{C\times H\times W}$, we define the two-dimensional DCT as 
\begin{align*}
    D_2[\mathbf{x}](c,h,w) = D_{1,\text{col}}[D_{1,\text{row}}[\mathbf{x}](c,h,\cdot)](c,\cdot,w), \;\;
    1\leq c\leq  C, 1\leq h \leq H, 1\leq w \leq W,
\end{align*}
where $D_{1,\text{row}}$ and $D_{1,\text{col}}$ are row-wise and column-wise one-dimensional DCT, respectively. 
The representation matrix of $D_{1,\text{row}}$, $D_{1,\text{col}}$, and their combination $D_2=D_{1,\text{col}}\circ D_{1,\text{row}}$ are all orthogonal. In fact, we have

\begin{theorem}
$D_{1,\text{row}}$, $D_{1,\text{col}}$ and $D_2=D_{1,\text{col}}\circ D_{1,\text{row}}$ are all orthogonal.
\end{theorem}

\begin{proof}
$D_{1,\text{row}}$, $D_{1,\text{col}}$ and $D_2 = D_{1,\text{col}}\circ D_{1,\text{row}}$ can be considered as linear transformation on the $\mathbb{R}^{C\times H\times W}$. Therefore, they have corresponding representation matrices (we also denote them by $D_{1,\text{row}}$, $D_{1,\text{col}}$ and $D_2$) that belong to $\mathbb{R}^{C\times H\times W}\times \mathbb{R}^{C\times H\times W}$.

According to the definition of $D_{1,\text{row}}$, it acts on multiple coordinate groups $\{(c,j,k),1\leq j \leq W\}$ separately, where $1\leq c \leq C$ and $1\leq k \leq H$. Additionally, for each coordinate group, it acts exactly as one-dimensional DCT, which is orthogonal. Therefore, $D_{1,\text{row}}$ is a block orthogonal matrix, and then orthogonal.
Similarly, $D_{1,\text{col}}$ is also a block orthogonal matrix, and then orthogonal. As a result, $D_2$ is orthogonal, as it is the matrix multiplication of two orthogonal matrices.
\end{proof}
\begin{remark}
According to Fubini's theorem, $D_{1,\text{row}}$ and $D_{1,\text{col}}$ are commutable (both of them can be considered as a discrete integrate operator along different axis respectively). Therefore, the order of $D_{1,\text{row}}$ and $D_{1,\text{col}}$ does not matter.
\end{remark}

\textbf{Discrete Fourier transform (DFT). }
The one-dimensional Discrete Fourier transform (DFT) is defined as
\begin{align*}
    D_{F,1}[\mathbf{x}](k) = \sum_{n=0}^{d-1}\mathbf{x}(n)e^{-\frac{i2\pi}{N}kn},k=0,\dots,d-1.
\end{align*}
Similarly, as DCT, we can define the corresponding two-dimensional DFT for image transformation.

\subsection{Automated parameter calculation}\label{sec:freq_analysis}
There are two parameters in $M_{\text{freq}}$ (\ie, $r$ and $\lambda$) in our TDAS (see Section 4.3 of the main paper). 
In this section, we describe how they can be estimated automatically from 
the statistics of the target distribution and the vanilla SGMs,
rather than manually tuned with less optimality.

The key idea is that performance degradation of the vanilla SGMs under acceleration can be quantified as some deviation in the frequency domain. 
This quantification is insightful and useful in how to rectify this deviation via properly setting the parameters of $M_{\text{freq}}$.

\begin{figure}[h]
\vspace{-4ex}
   \begin{subfigure}{0.32\linewidth}
   \includegraphics[scale=0.26]{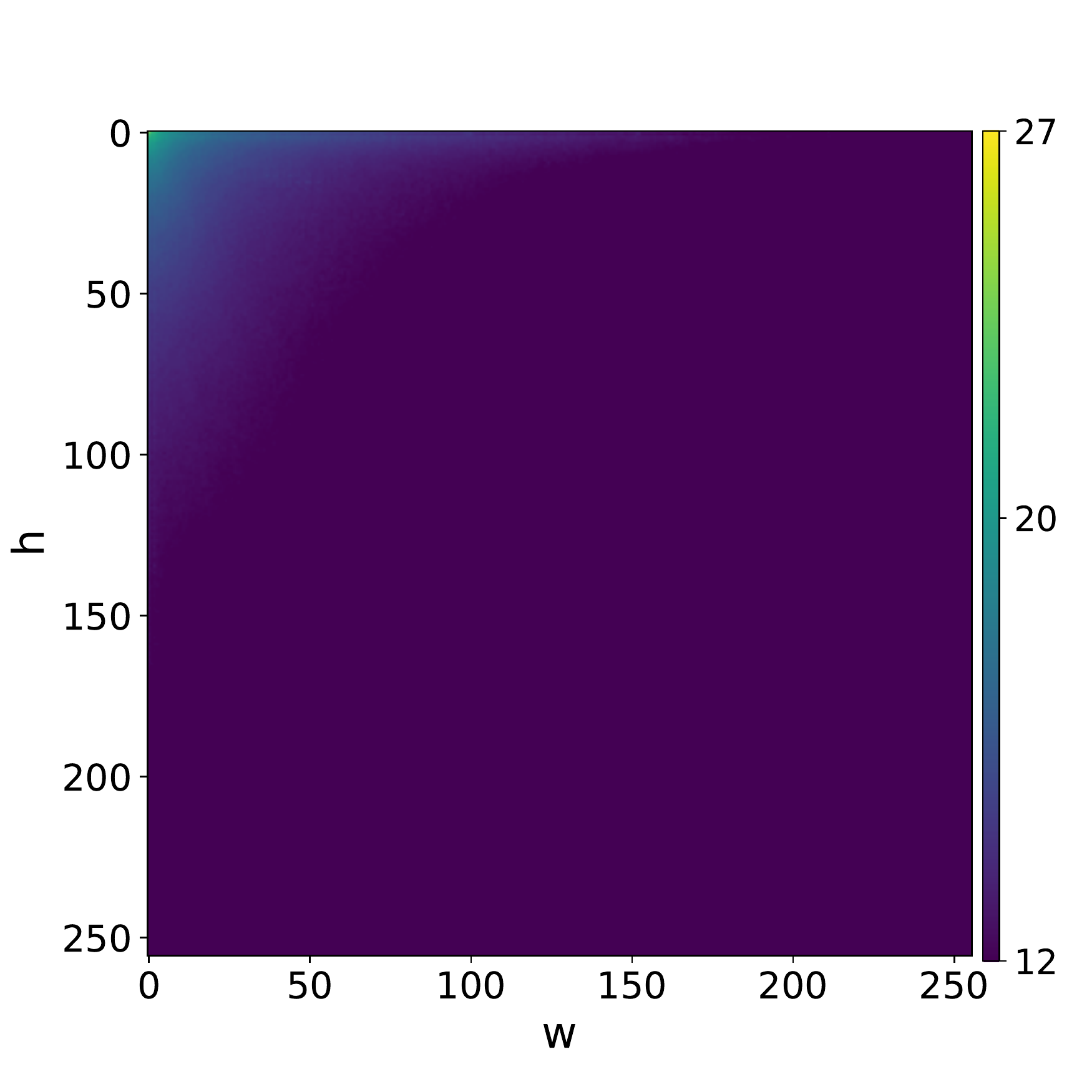}
   \caption{}
   \end{subfigure}
   \begin{subfigure}{0.32\linewidth}
    \includegraphics[scale=0.26]{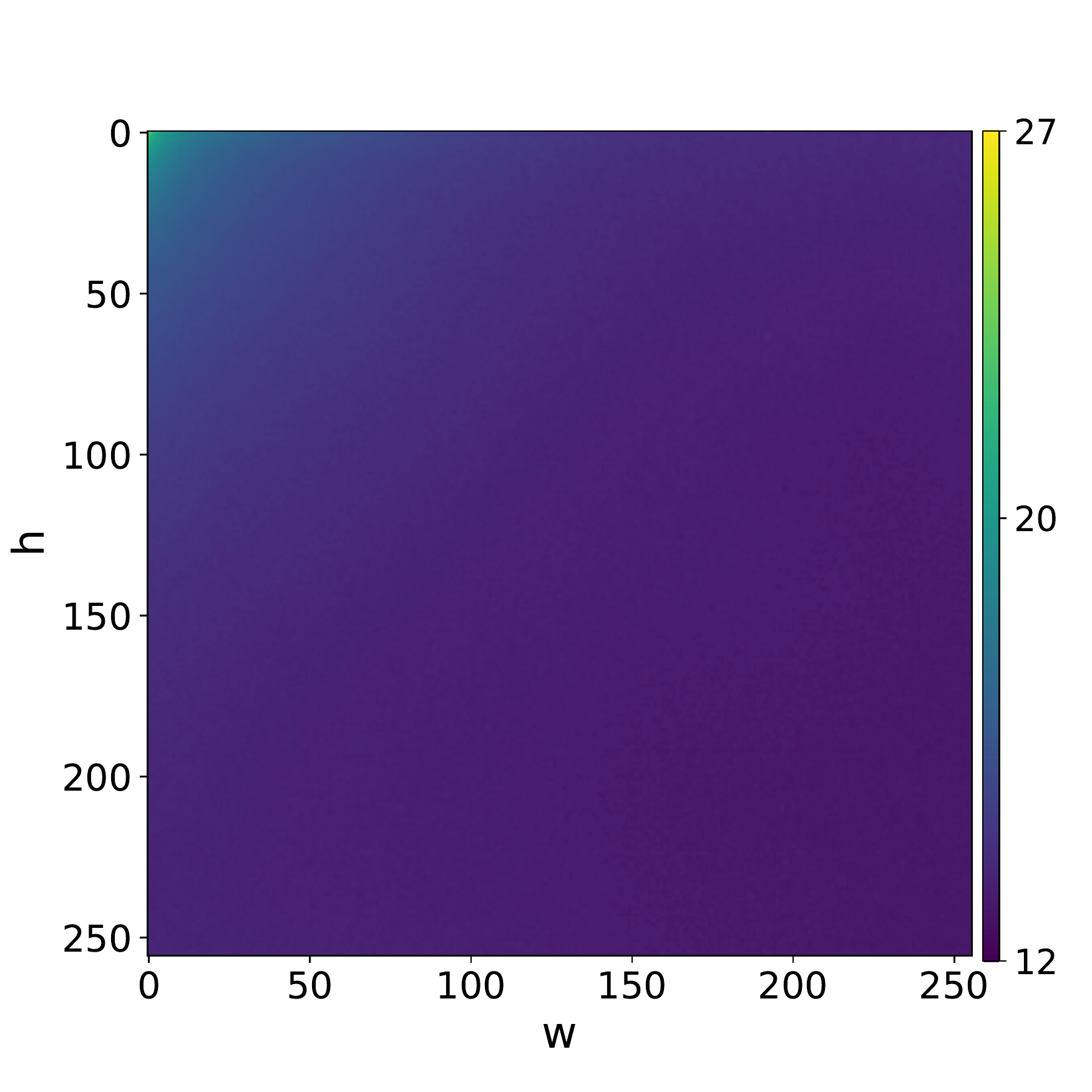}
    \caption{}
   \end{subfigure}
   \begin{subfigure}{0.32\linewidth}
    \includegraphics[scale=0.26]{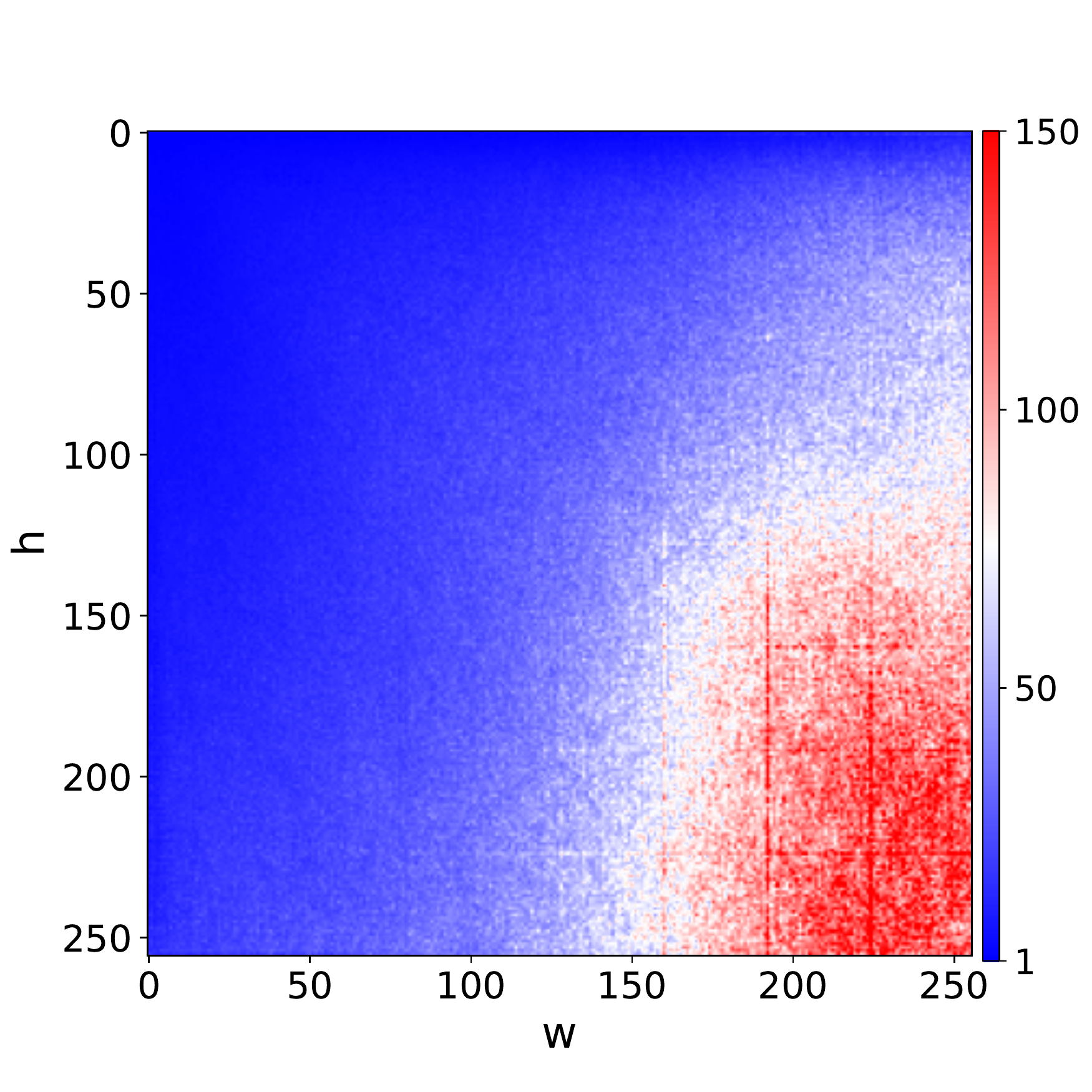}
    \caption{}
   \end{subfigure}
    \caption{The frequency statistics of LSUN (bedroom) images 
    generated by NCSN++~\cite{song2020score} using DCT.  
    (a) The frequency distribution function of the images generated without acceleration ($\phi_{p^\ast}$). 
    (b) The frequency distribution function of the images generated under $5\times$ acceleration ($\phi_{{p}'}$). 
    (c) Their ratio function ($\gamma_{{p}',p^\ast}$).
    $w$ (x-axis) and $h$ (y-axis) denote the width and height coordinate of image, respectively. 
    Image resolution: 256$\times$256 in pixel.
    We have taken the natural logarithm for the frequency distribution for better demonstration.
    }
    \label{fig:bedroom_freq}
\end{figure}

Take LSUN (bedroom)~\cite{yu2015lsun} with $256\times 256$ resolution for a demonstration. Suppose we want to accelerate the NCSN++~\cite{song2020score} up to $5\times$ (400 iterations). Direct acceleration leads to severe deterioration of the outputs as shown in Figure 3 of the main paper. 
For quantization, we estimate the frequency statistics (Figure~\ref{fig:bedroom_freq}(left)) of the images generated by NSCN++ {\em without} acceleration as:
\begin{align*}
    \phi_{p^\ast}(h,w) = \frac{1}{C}\sum_{c=1}^C\mathbb{E}_{\mathbf{x}\sim p^\ast(\mathbf{x})}\left [ D[\mathbf{x}]\odot D[\mathbf{x}] \right](h,w,c),1\leq h \leq H, 1\leq w \leq W,
\end{align*}
where $D[\cdot]$ means discrete cosine transform (DCT).
Similarly, we compute the counterpart (Figure~\ref{fig:bedroom_freq} Middle) of
the images generated by accelerated NCSN++ as:
\begin{align*}
    \phi_{{p}'}(h,w) = \frac{1}{C}\sum_{c=1}^C\mathbb{E}_{\mathbf{x}\sim {p}'(\mathbf{x})}\left [ D[\mathbf{x}]\odot D[\mathbf{x}] \right](h,w,c),1\leq h \leq H, 1\leq w \leq W,
\end{align*}
where ${p}'$ denotes the distribution of the images generated from accelerated NCSN++. 

In general,
we denote $\phi_{p}$ as the frequency distribution function of the distribution $p$.
Then the deviation of the generated images from the dataset can be defined as their ratio as
\begin{align}\label{eq:ratio}
    \gamma_{{p}',p^\ast}(h,w) = \frac{\phi_{{p}'}(h,w)}{\phi_{p^\ast}(h,w)},1\leq h \leq H, 1\leq w \leq W.
\end{align}
As shown in Figure~\ref{fig:bedroom_freq}(right), the high-frequency part of the generated images under acceleration is dramatically higher than that without acceleration.

Denote $S_{{p}',p^\ast} = \{\gamma_{{p}',p^\ast}(h,w),1\leq h \leq H, 1\leq w \leq W\}$ and the $\alpha$ quantile of a finite set $S$ as $Q_{\alpha}(S) = \min_{x\in S}(\{x\in S,\#\{y\in S,y\leq x\}\geq \alpha \#S\})$ where $\#$ means the cardinality of a set.
To rectify this excessive high-frequency amplitude, we set the parameter $M_{\text{freq}}$ as:
\begin{align}\label{eq: freq_mask2}
M_{\text{freq}}(c,h,w) = 
\left\{\begin{matrix}
 1&,& d_0(h,w) \leq 2r_{1}^2,\\
\lambda_1&,&2r_{1}^2 <  d_0(h,w) \leq 2r_{2}^2\\
\lambda_2&,& d_0(h,w) > 2r_{2}^2,\\
\end{matrix}\right.,
\end{align}
where $d_0(h,w) = h^2+w^2$ denotes the square distance from $(h,w)$ to $(0,0)$ (the point corresponding the lowest frequency part)
and
\begin{align*}
    &\lambda_1 = \frac{\text{ave}(S_{{p}',p^\ast})}{Q_{0.75}(S_{{p}',p^\ast})},\\
    &\lambda_2 =\frac{\text{ave}(S_{{p}',p^\ast})}{Q_{0.9}(S_{{p}',p^\ast})},
\end{align*}
\begin{align*}
    &r_1 =\arg\min_{r}\{r>0,\text{ave}(\{\gamma_{{p}',p^\ast}(h,w),d_0(h,w)\geq 2r^2,1\leq h \leq H, 1\leq w \leq W\})=Q_{0.75}(S_{{p}',p^\ast})\},\\
    &r_2 =\arg\min_{r}\{r>0,\text{ave}(\{\gamma_{{p}',p^\ast}(h,w),d_0(h,w)\geq 2r^2,1\leq h \leq H, 1\leq w \leq W\})=Q_{0.9}(S_{{p}',p^\ast})\},
\end{align*}
where $\text{ave}(\cdot)$ means the average of an finite set.
Our intuition is explained as follows. 
We consider $\frac{Q_{0.75}(S_{{p}',p^\ast})}{\text{ave}(S_{{p}',p^\ast})}$ and $\frac{Q_{0.9}(S_{{p}',p^\ast})}{\text{ave}(S_{{p}',p^\ast})}$ as the middle and high level deviation ratio of the generated images distribution under acceleration, respectively. 
To counter their effect, we should scale them by a degree of the {\em inverse}.
Hence $\lambda_1$ and $\lambda_2$ are set to the inverse of the respective deviation ratio.
On the other hand, $r_1$ and $r_2$ specify the radius of the sectors outside of which the rectification will be imposed on (Figure~\ref{fig:freq_mask2}).
\begin{figure}[h]
    \centering
    \includegraphics[scale=0.5]{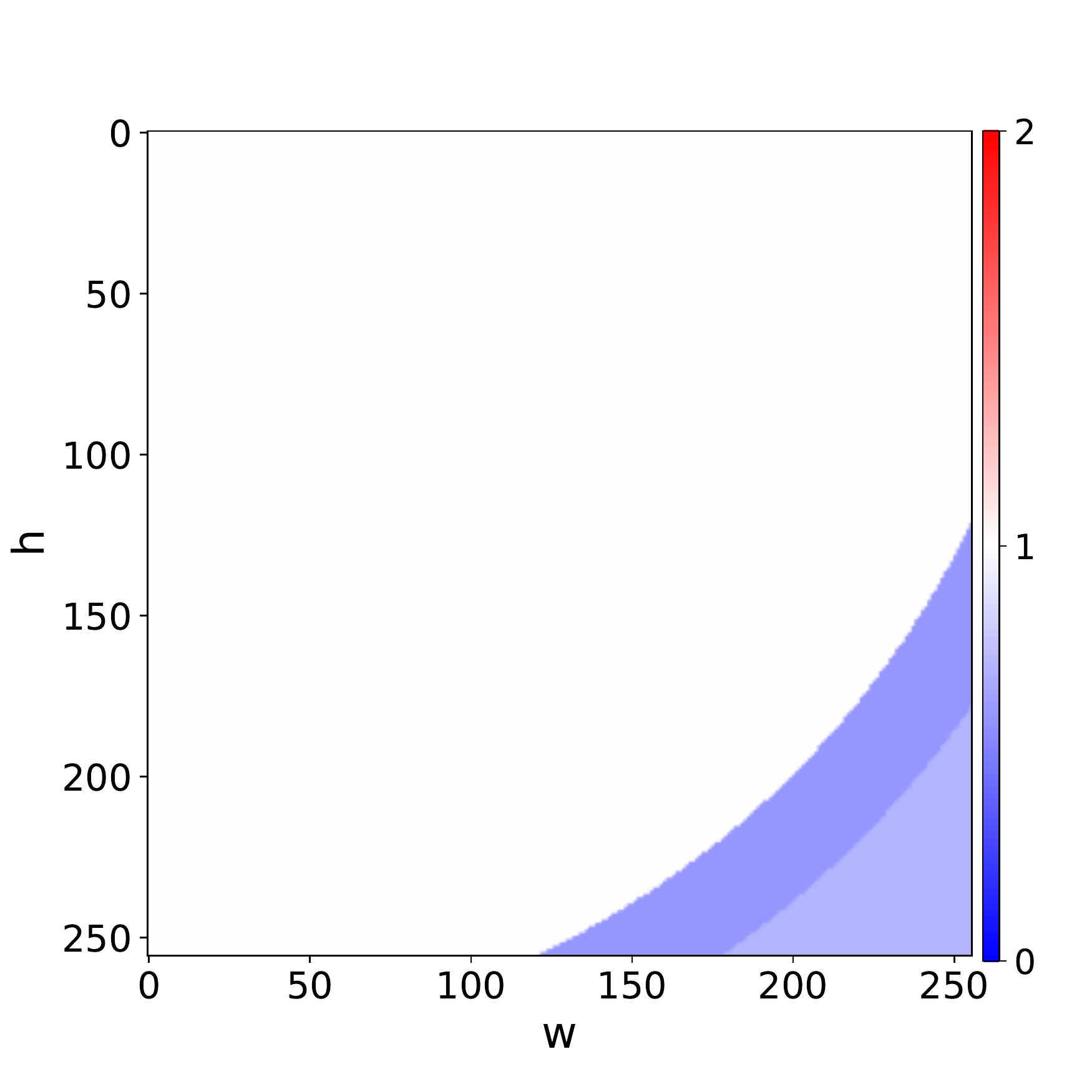}
    \vspace{-3ex}
    \caption{
    Visualization of TDAS's parameter $M_{\text{freq}}$ for accelerating NCSN++~\cite{song2020score} to generate LSUN (bedroom) images at a resolution of 256$\times$256.
    $w$ (x-axis) and $h$ (y-axis) denote the width and height coordinate of image, respectively.
    Here, the parameters estimated using DCT are: $\lambda_1=0.59,\lambda_2=0.71,r_1=0.78H$ and $r_2 = 0.86H$.}
    \label{fig:freq_mask2}
\end{figure}

\textbf{Calculating filters using DFT. }
Instead of DCT, we can also apply the DFT to construct the frequency filer (compared to Eq. (5) in the main paper) as:
\begin{align*}
        \bm{\eta}_t = \text{Real}\big[D_F^{-1}\big[M_{\text{freq}}\odot D_F[M_{\text{space}}\odot \mathbf{z}_t]\big]\big],
\end{align*}
where $D_F[\cdot]$ is the DFT, and $\text{Real}[\cdot]$ means taking the real part. The corresponding frequency statistics (Figure~\ref{fig:bedroom_freq_fct}) can be defined as
\begin{align*}
    &\phi_{p^\ast}(h,w) = \frac{1}{C}\sum_{c=1}^C\mathbb{E}_{\mathbf{x}\sim p^\ast(\mathbf{x})}\left [ D_F[\mathbf{x}]\odot \overline{D_F[\mathbf{x}]} \right](h,w,c),1\leq h \leq H, 1\leq w \leq W,\\
    &\phi_{{p}'}(h,w) = \frac{1}{C}\sum_{c=1}^C\mathbb{E}_{\mathbf{x}\sim {p}'(\mathbf{x})}\left [D_F[\mathbf{x}]\odot \overline{D_F[\mathbf{x}]} \right](h,w,c),1\leq h \leq H, 1\leq w \leq W,\\
    &\gamma_{{p}',p^\ast}(h,w) = \frac{\phi_{{p}'}(h,w)}{\phi_{p^\ast}(h,w)},1\leq h \leq H, 1\leq w \leq W,
\end{align*}
where $\overline{\cdot}$ means taking element-wise complex conjugation.

\begin{figure}[h]
\vspace{-4ex}
   \begin{subfigure}{0.32\linewidth}
   \includegraphics[scale=0.26]{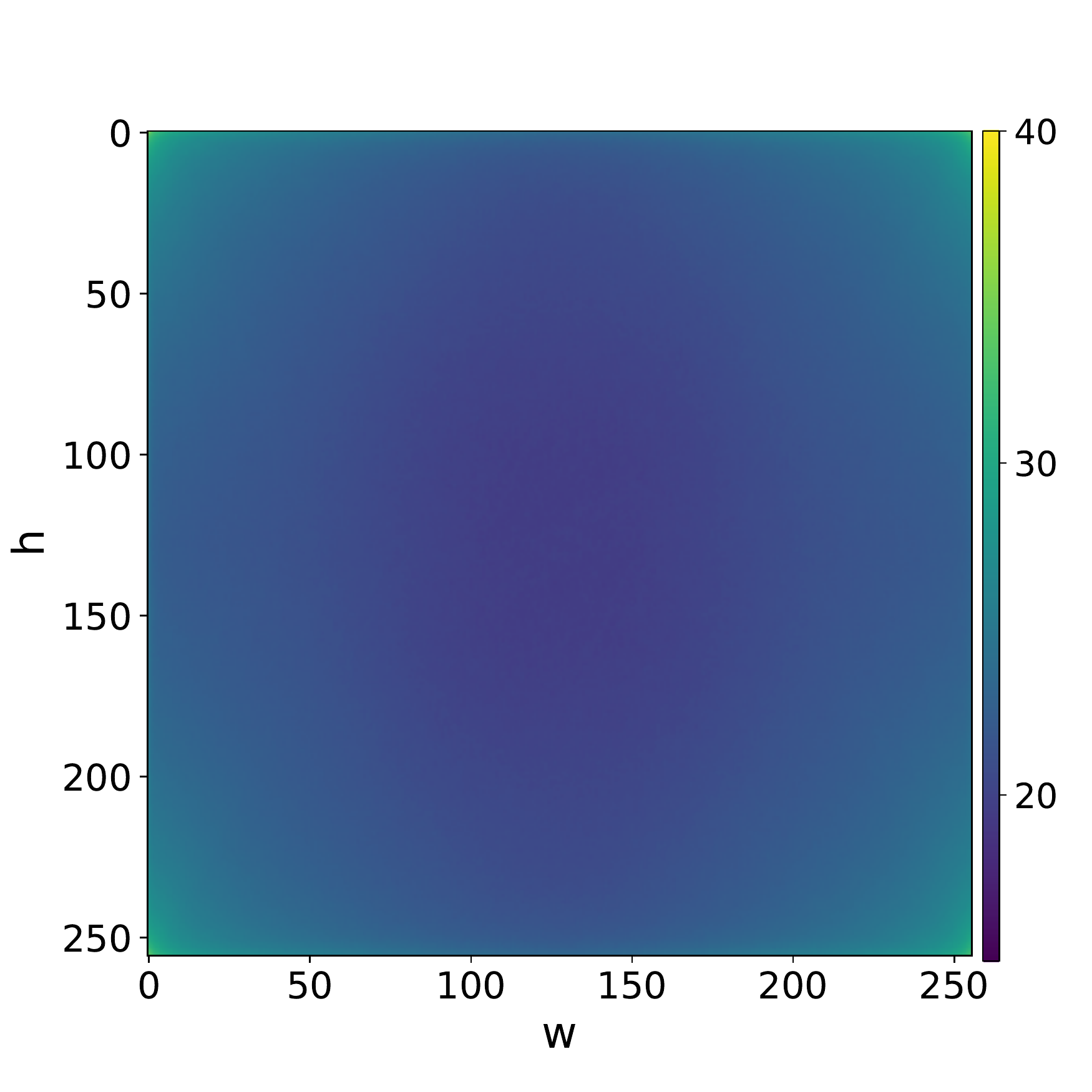}
   \caption{}
   \end{subfigure}
   \begin{subfigure}{0.32\linewidth}
    \includegraphics[scale=0.26]{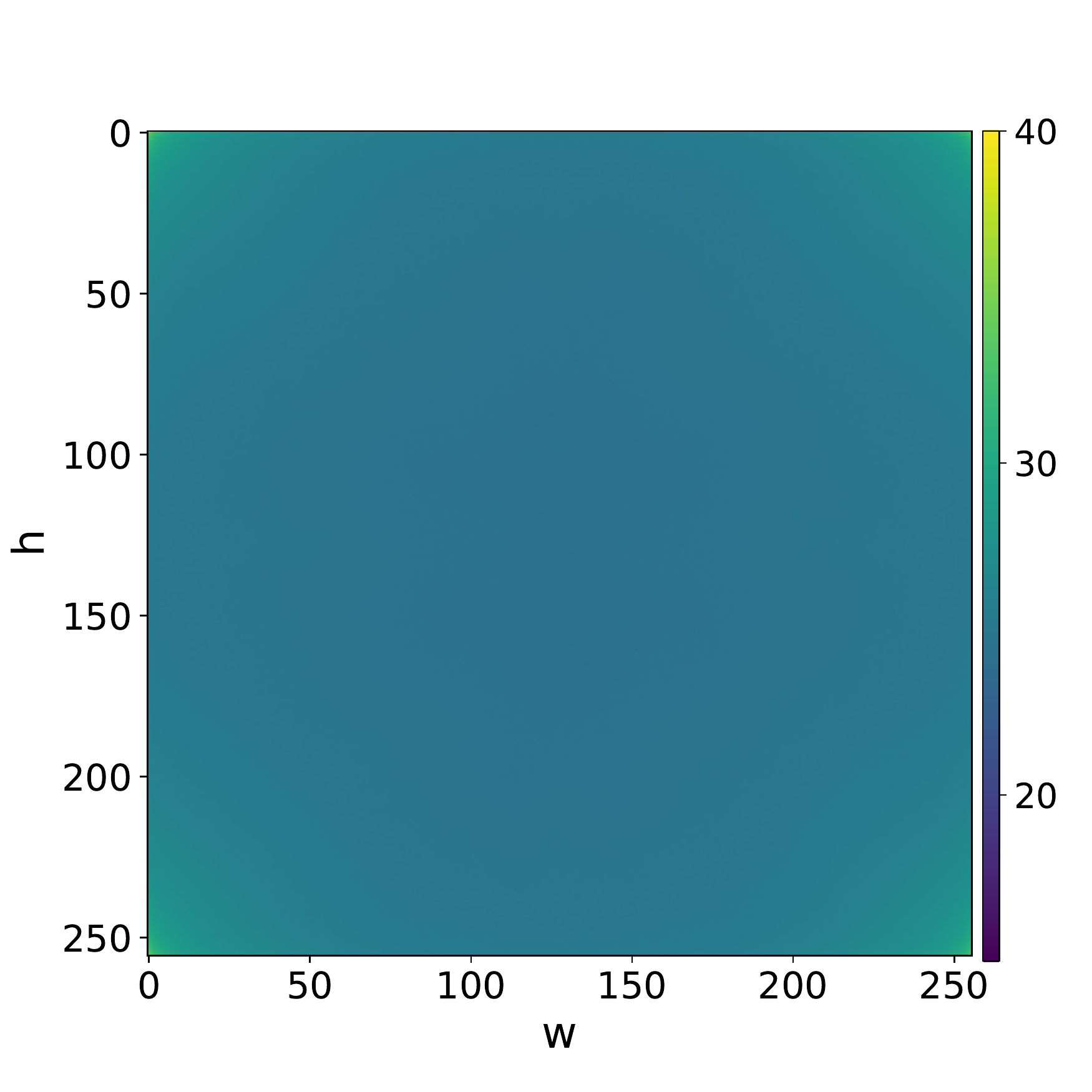}
    \caption{}
   \end{subfigure}
   \begin{subfigure}{0.32\linewidth}
    \includegraphics[scale=0.26]{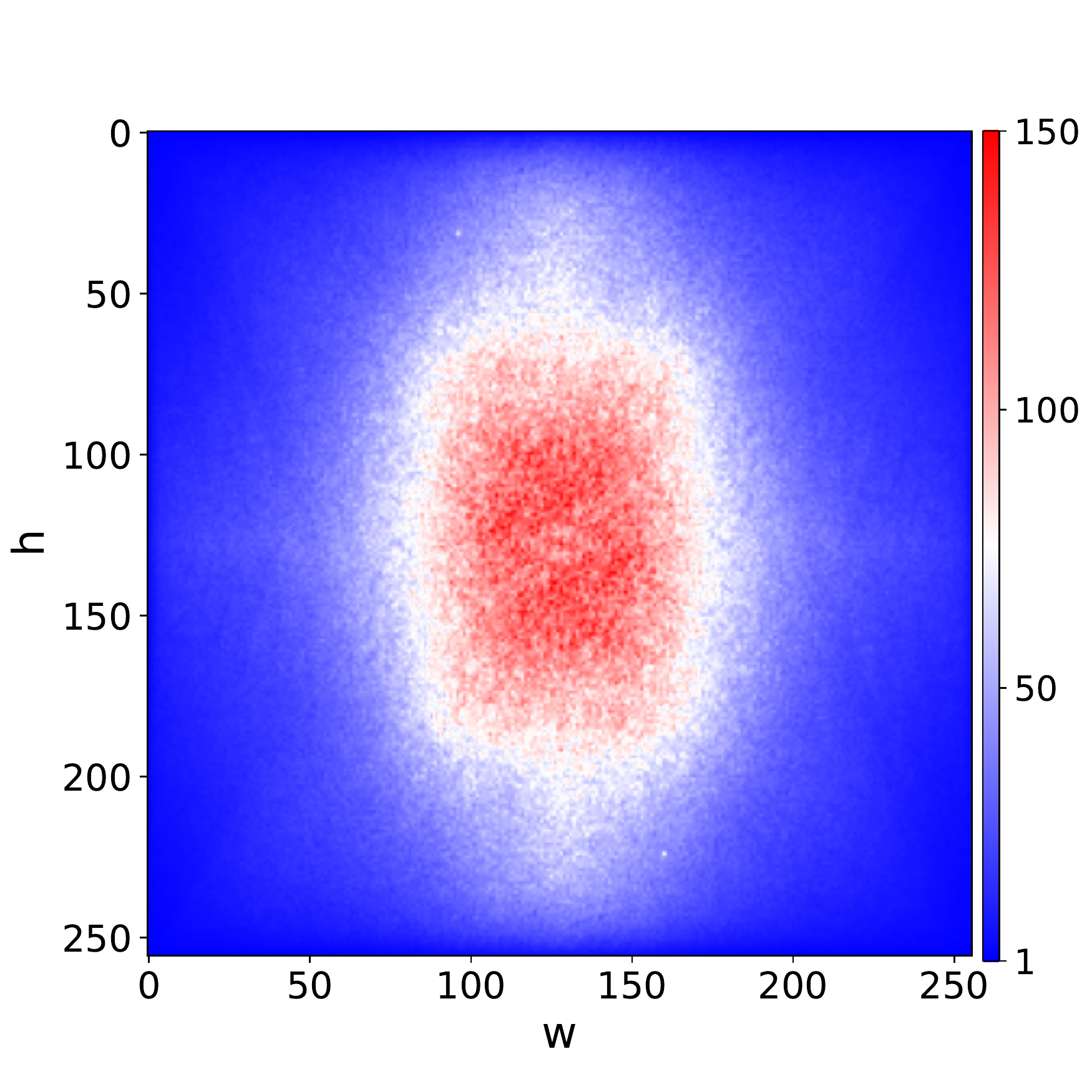}
    \caption{}
   \end{subfigure}
   \caption{The frequency statistics of LSUN (bedroom) images 
    generated by NCSN++~\cite{song2020score} using DFT.  
    (a) The frequency distribution function of the images generated without acceleration ($\phi_{p^\ast}$). 
    (b) The frequency distribution function of the images generated under $5\times$ acceleration ($\phi_{{p}'}$). 
    (c) Their ratio function ($\gamma_{{p}',p^\ast}$).
    $w$ (x-axis) and $h$ (y-axis) denote the width and height coordinate of image, respectively. 
    Image resolution: 256$\times$256 in pixel.
    We have taken the natural logarithm for the frequency distribution for better demonstration.
    }
    \label{fig:bedroom_freq_fct}
\end{figure}

\begin{figure}[ht]
    \vspace{-5ex}
    \centering
    \includegraphics[scale=0.5]{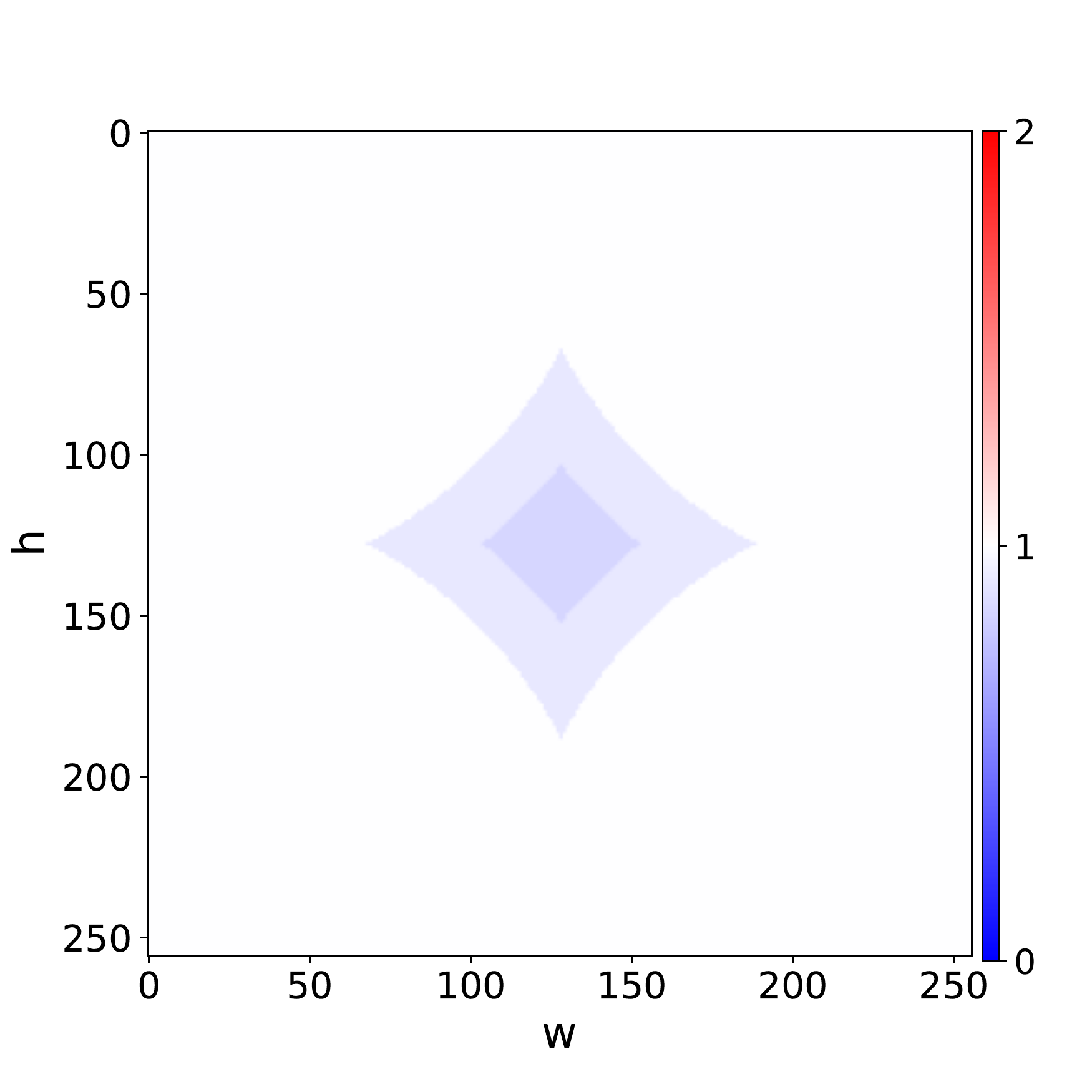}
    \vspace{-3ex}
    \caption{
    Visualization of TDAS's parameter $M_{\text{freq}}$ for accelerating NCSN++~\cite{song2020score} to generate LSUN (bedroom) images at a resolution of 256$\times$256.
    $w$ (x-axis) and $h$ (y-axis) denote the width and height coordinate of image, respectively.
    Here, the parameters estimated using DFT are: $\lambda_1=0.907,\lambda_2=0.842,r_1=0.0.04H$ and $r_2 = 0.0455H$.}
    \label{fig:freq_mask}
\end{figure}

We similarly estimate the corresponding parameters ($\lambda_1,\lambda_2,r_1,r_2$) of $M_{\text{freq}}$,
given a target task, an existing SGM, and a desired iteration number. The only difference is that in this case the square distance to $(0,0)$ (\ie, $d_0(h,w)$ in Eq.~\eqref{eq: freq_mask2}) is replaced by 
\begin{align*}
  \hat{d}_0(h,w) = \min\{ h^2+w^2,(H-h)^2+w^2,h^2+(W-w)^2,(H-h)^2+(W-w)^2 \}, 
\end{align*}
due to the symmetry property of the Fourier spectrum.

We summarize the resulting parameters for NCSN++ in Table~\ref{tab:para}. We have the following findings: (1) As the iterations decrease, both $\lambda_1$ and $\lambda_2$ decrease, suggesting that larger step size leads to more high-frequency noises generated by NCSN++. (2) For the cases of DCT, both $r_1$ and $r_2$ remain the same across different tasks and iterations, implying that there is some invariant properties for NCSN++ in generating images, such as NCSN++ could be easily confused on certain area in the frequency domain.
Note, in the main paper, by default we estimate the parameters using DFT, except LSUN datasets where we find empirically that DCT is a better choice.
The reason behind this phenomenon needs a further study.

\begin{table}[ht]
\caption{The calculated parameters of $M_{\text{freq}}$ for NCSN++~\cite{song2020score} on different tasks. }
\label{tab:para}
\centering
\begin{tabular}{cccccccc}
\toprule[2pt]
\textbf{Dataset}                   &\textbf{Resolution}            & \textbf{Iterations} &\textbf{Method}& $\lambda_1$ & $\lambda_2$ & $r_1$ & $r_2$  \\ \midrule
\multirow{7}{*}{CIFAR-10} & \multirow{7}{*}{32x32}& 200        &DFT        & 0.984     & 0.967     & 0.400  & 0.455 \\
                           &                      & 166        &DFT        & 0.973     & 0.945     & 0.400  & 0.455 \\
                           &                      & 142        &DFT        & 0.965     & 0.932     & 0.400  & 0.455 \\
                           &                      & 125        &DFT        & 0.959     & 0.929     & 0.400  & 0.455 \\
                           &                      & 111        &DFT        & 0.952     & 0.901     & 0.400  & 0.455 \\
                           &                      & 100        &DFT        & 0.948     & 0.895     & 0.400  & 0.455 \\
                           &                      & 50         &DFT        & 0.892     & 0.798     & 0.400  & 0.455 \\\midrule
CelebA                           &64x64          & 132        &DFT        & 0.877     & 0.782     & 0.398  & 0.456 \\\midrule
LSUN (church)                    &256x256        & 400        &DFT        & 0.921     & 0.873     & 0.400  & 0.455 \\\midrule
LSUN (church)                    &256x256        & 400        &DCT        & 0.591     & 0.511     & 0.781  & 0.862 \\\midrule
LSUN (bedroom)                   &256x256        & 400        &DFT        & 0.907     & 0.842     & 0.400  & 0.455 \\\midrule
LSUN (bedroom)                   &256x256        & 400        &DCT        & 0.638     & 0.540     & 0.770  & 0.901 \\\midrule
FFHQ                             &1024x1024      & 100        &DFT        & 0.955     & 0.878     & 0.400  & 0.455 \\\bottomrule[2pt]

\end{tabular}
\end{table}

\textbf{Computational cost. }
The frequency statistics for parameter estimation
is efficient computationally.
Only 200 samples per task need to be generated
by the SGMs under each acceleration rate.

\textbf{Parameter sensitivity. }
We investigate the parameter sensitivity of the frequency filter (Eq.~\eqref{eq: freq_mask2}). We use NCSN++ with the sampling iterations $T=400$ on LSUN (bedroom) calculating parameters by DCT (DFT version has the similar results). 

We first vary $\lambda_1$ and $\lambda_2$, whilst fixing $r_{1}$ and $r_{,2}$. It is observed in Figure~\ref{fig: para1} that there is a large range for both parameters that can produce samples of high quality. If $\lambda_1$ and $\lambda_2$ are too high, TDAS degrades to the vanilla sampling method, and the images generated corrupt;
On the contrary, if they are too low, which means that we shrink the high-frequency part too much, detailed structures will fade away, resulting in pale images.

We then vary $r_{1}$ and $r_{2}$, whilst fixing $\lambda_1$ and $\lambda_2$.
As shown in Figure~\ref{fig: para2},
there is still a large range for both parameters that can produce images of high quality.
It is observed that as the two threshold radiuses decrease, the samples gradually lose detailed structures. On the one hand, if $r_{1}$ and $r_{2}$ are too high, TDAS degrades to the vanilla sampling method, and the samples are disturbed by high-frequency noise dramatically. 
In contrast, if $r_{1}$ and $r_{2}$ are too low, the low-frequency appearance of an image generated will then disappear, resulting in tedious images without textures and objects.

\textbf{Space filter. }
We use the space filter $M_{\text{space}}$ for CIFAR-10, CelebA~\cite{liu2015faceattributes} and FFHQ~\cite{karras2019style}, since they have an obvious spatial distribution characteristics. $M_{\text{space}}$ is calculated from the mean of the datasets as shown in Section 4.3 of the main paper.
Empirically, we find out that it is better to normalize the space filter to guarantee the stability. Concretely, we apply the following normalized space filter
\begin{align*}
  \hat{M}_{\text{space}}(c,h,w) \leftarrow  \frac{1}{3}\Big(2\frac{M_{\text{space}}(c,h,w)}{\max_{c,h,w}M_{\text{space}}(c,h,w)}+1 \Big).
\end{align*}

\subsection{Accelerating the DDPMs}\label{sec:ddpm}

\begin{figure}[h]
\begin{subfigure}{0.48\linewidth}
\includegraphics[scale=0.285]{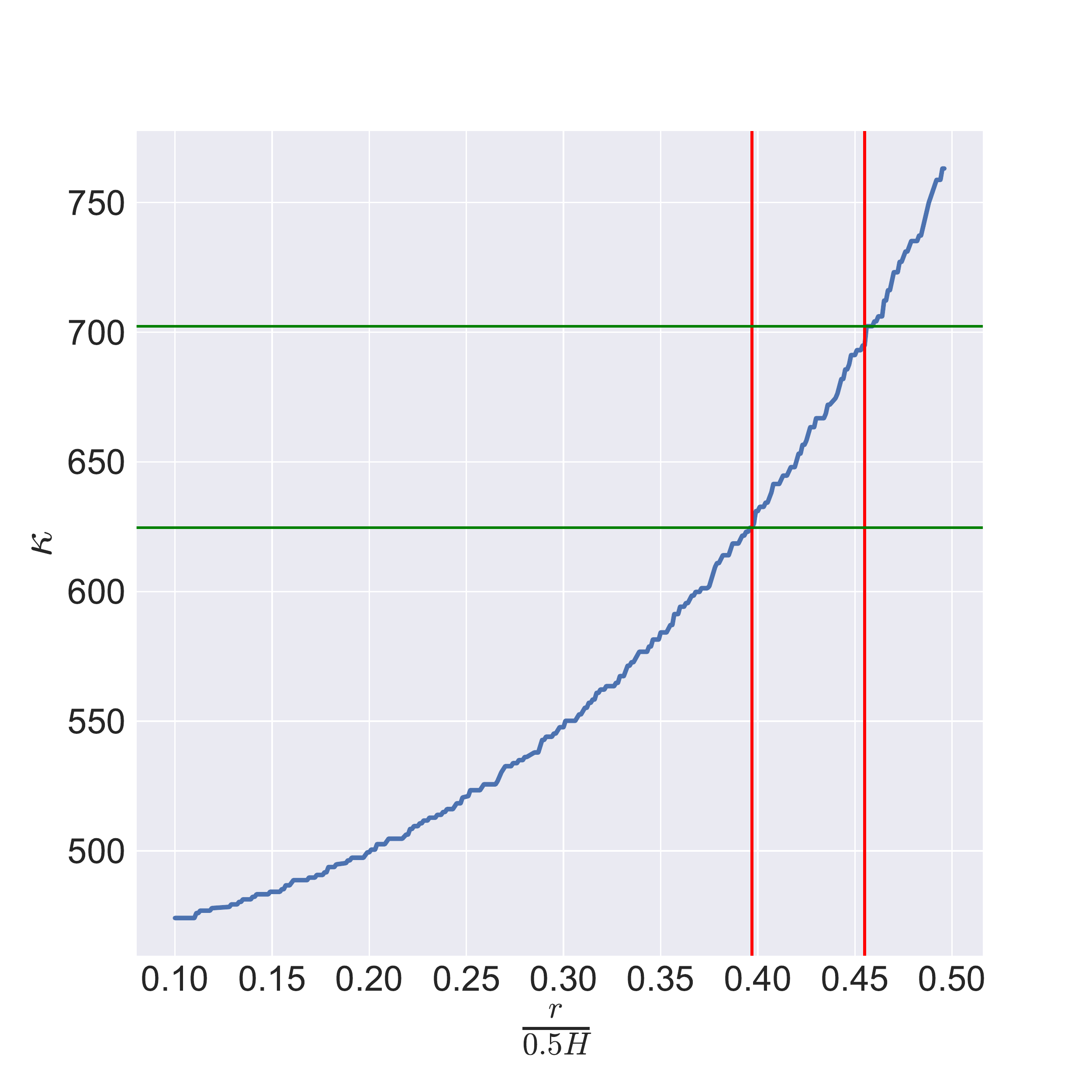}
\caption{}
\end{subfigure}
\hspace{1em}
\begin{subfigure}{0.48\linewidth}
\includegraphics[scale=0.285]{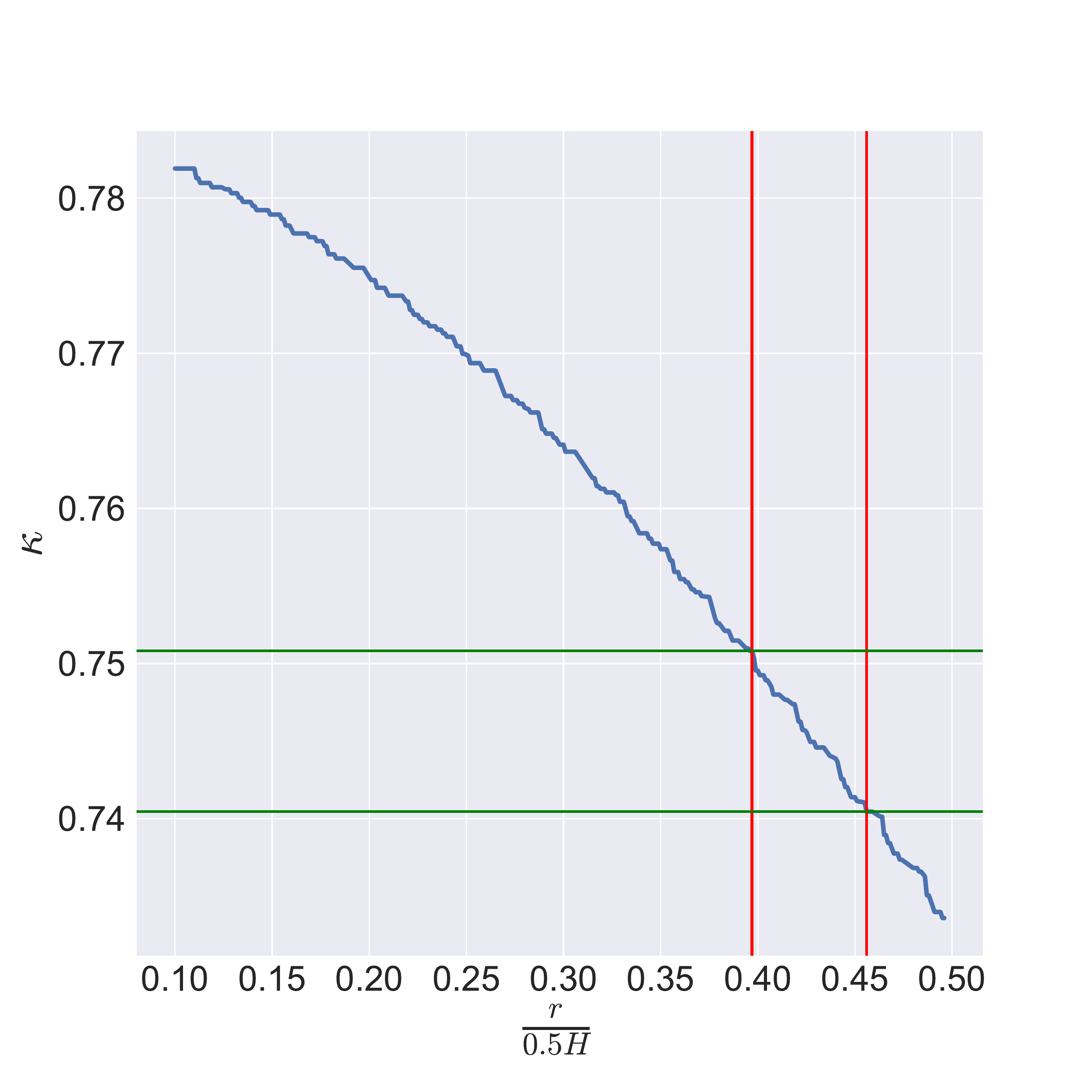}
\caption{}
\end{subfigure}
\caption{The frequency behavior functions of NCSN++~\cite{song2020score} (a) and Analytic-DDIM~\cite{bao2022analytic} (b).
The green lines label the levels of $Q_{0.75}(S_{{p}',p^\ast}),Q_{0.9}(S_{{p}',p^\ast})$ for NCSN++ and $Q_{0.25}(S_{{p}',p^\ast}),Q_{0.1}(S_{{p}',p^\ast})$ for Analytic-DDIM.
The red lines point to the values of $r$ where $\kappa(r)$ can reach the corresponding level. 
Dataset: CelebA ($64\times64$).}\label{fig:kappa}
\end{figure}

In this section, 
we generalize the application of our TDAS to
the Denosing Diffusion Probabilistic Models (DDPMs)~\cite{DBLP:conf/nips/HoJA20,song2020denoising,DBLP:conf/icml/NicholD21,bao2022analytic}, a class of generative models that can be considered as a variant of SGMs~\cite{song2020score}.

With the frequency statistics in Section~\ref{sec:freq_analysis}, we have defined the ratio function $ \gamma_{{p}',p^\ast}$ (Eq.~\eqref{eq:ratio}). To further quantify the behavior of a SGM across different frequencies, we define the following {\em frequency behavior function} as:
\begin{align*}
    \kappa(r) = \text{ave}(\{\gamma_{{p}',p^\ast}(h,w),d_0(h,w)\geq 2r^2,1\leq h \leq H, 1\leq w \leq W\}),0\leq r \leq \sqrt{W^2+H^2}.
\end{align*}
This aims to describe how a SGM amplifies or shrinks the corresponding area across the frequency dimension on a target task.
It is observed in Figure ~\ref{fig:kappa} (left) that the behavior function of NCSN++ on CelebA ($64\times64$) is monotonically increasing and $\kappa>1$, suggesting that NCSN++ tends to amplify the high-frequency part of the output images, leading to high-frequency noises.

This also applies to DDPMs.
Taking a recent DDPM model called Analytic-DDIM~\cite{bao2022analytic} as an example, we calculate the corresponding frequency behavior function as shown in Figure~\ref{fig:kappa}(right).
On the contrary, this $\kappa$ is monotonically decreasing and $\kappa<1$, suggesting that Analytic-DDIM tends to shrink the high-frequency part of the output images. 
This finding implies that NCSN++ and Analytic-DDIM may exhibit opposite behavior intrinsically.
According to \cite{song2020score}, DDPMs and SGMs correspond to variance preserving (VP) and variance exploding (VE) stochastic different equation (SDE), respectively. Our above findings indicate that VE and VP SDEs might generally have opposite behaviors in the frequency domain in image generation. 
This is inspiring for developing more proper SDEs to 
better implement the diffusion process.

As shown in Figure~\ref{fig:kappa}, another interesting observation is that near $r=0.4$, NCSN++ reaches its $Q_{0.75}(S_{{p}',p^\ast})$ and Analytic-DDIM reaches its $Q_{0.25}(S_{{p}',p^\ast})$; Meanwhile, near $r=0.455$, NCSN++ reaches its $Q_{0.9}(S_{{p}',p^\ast})$, and Analytic-DDIM reaches its $Q_{0.1}(S_{{p}',p^\ast})$.
This suggests that the deviation patterns caused by NCSN++ and Analytic-DDIM share structural similarity, despite in the opposite direction.
{As what we do for SGMs,} to accelerate Analytic-DDIM on CelebA ($64\times64$), we can set the corresponding parameters of $M_{\text{freq}}$ as follows
\begin{align*}
    &\lambda_1 = \frac{\text{ave}(S_{{p}',p^\ast})}{Q_{0.25}(S_{{p}',p^\ast})},\\
    &\lambda_2 =\frac{\text{ave}(S_{{p}',p^\ast})}{Q_{0.1}(S_{{p}',p^\ast})},
\end{align*}
\begin{align*}
    &r_1 =\arg\min_{r}\{r>0,\kappa(r)=Q_{0.25}(S_{{p}',p^\ast})\},\\
    &r_2 =\arg\min_{r}\{r>0,\kappa(r)=Q_{0.1}(S_{{p}',p^\ast})\}.
\end{align*}
Note we apply $Q_{0.1}$ and $Q_{0.25}$ instead of $Q_{0.9}$ and $Q_{0.75}$ as used for NCSN++. This is due to the monotonical decreasing property of $\kappa(r)$ with Analytic-DDIM, which is in opposite to NCSN++.

It is shown in 
Table~\ref{tab:ddim_para} across a variety of acceleration rates (\ie, the sampling iterations), TDAS can consistently further improve the performance of Analytic-DDIM.
In this test, we find the space filter $M_{\text{space}}$ is not helpful and we will further investigate the issue.

\begin{table}[H]
\begin{center}
\caption{The estimated parameters of $M_{\text{freq}}$ by DFT for NCSN++~\cite{song2020score} on CelebA ($64\times64$). 
}\label{tab:ddim_para}
  \begin{tabular}{ccccc}
\toprule[2pt]
\textbf{Iterations} & $\lambda_1$ & $\lambda_2$ & $r_1$    & $r_2$    \\ \midrule
10        & 2.94    & 3.04    & 0.399 & 0.455 \\
25        & 1.70    & 1.74    & 0.399 & 0.455 \\
50        & 1.33    & 1.35    & 0.399 & 0.455 \\
100       & 1.14    & 1.16    & 0.399 & 0.455 \\
200       & 1.05    & 1.06    & 0.399 & 0.455 \\ \bottomrule[2pt]
\end{tabular}  
\end{center}

\end{table}

\begin{table}[H]
\begin{center}
\caption{
Effect of TDAS on improving Analytic-DDIM on CelebA ($64\times64$).
Metric: FID score.
}
\label{tab:ddim}
    \begin{tabular}{ccc}
\toprule[2pt]
\textbf{Iterations} & w/o TDAS & w/ TDAS \\ \midrule
10                 & 15.53             & \textbf{15.31}   \\
25                 & 9.42              & \textbf{9.15}    \\
50                 & 6.17              & \textbf{6.08}    \\
100                & 4.31              & \textbf{4.29}    \\
200                & 3.54              & \textbf{3.49}    \\ \bottomrule[2pt]
\end{tabular}
\end{center}

\end{table}

\subsection{More samples}\label{sec:sample}

We provide more samples generated by NCSN++ with TDAS for LSUN~\cite{yu2015lsun} (bedroom and church) and FFHQ~\cite{karras2019style}.
The results are presented in Figure~\ref{fig:bedroom_demo} and Figure~\ref{fig:church_demo}, Figure~\ref{fig:demo_sm} and Figure~\ref{fig:demo_sm2}.

\begin{figure}[ht]
\vspace{-14ex}
	\begin{center}
    \centerline{\includegraphics[scale=0.42]{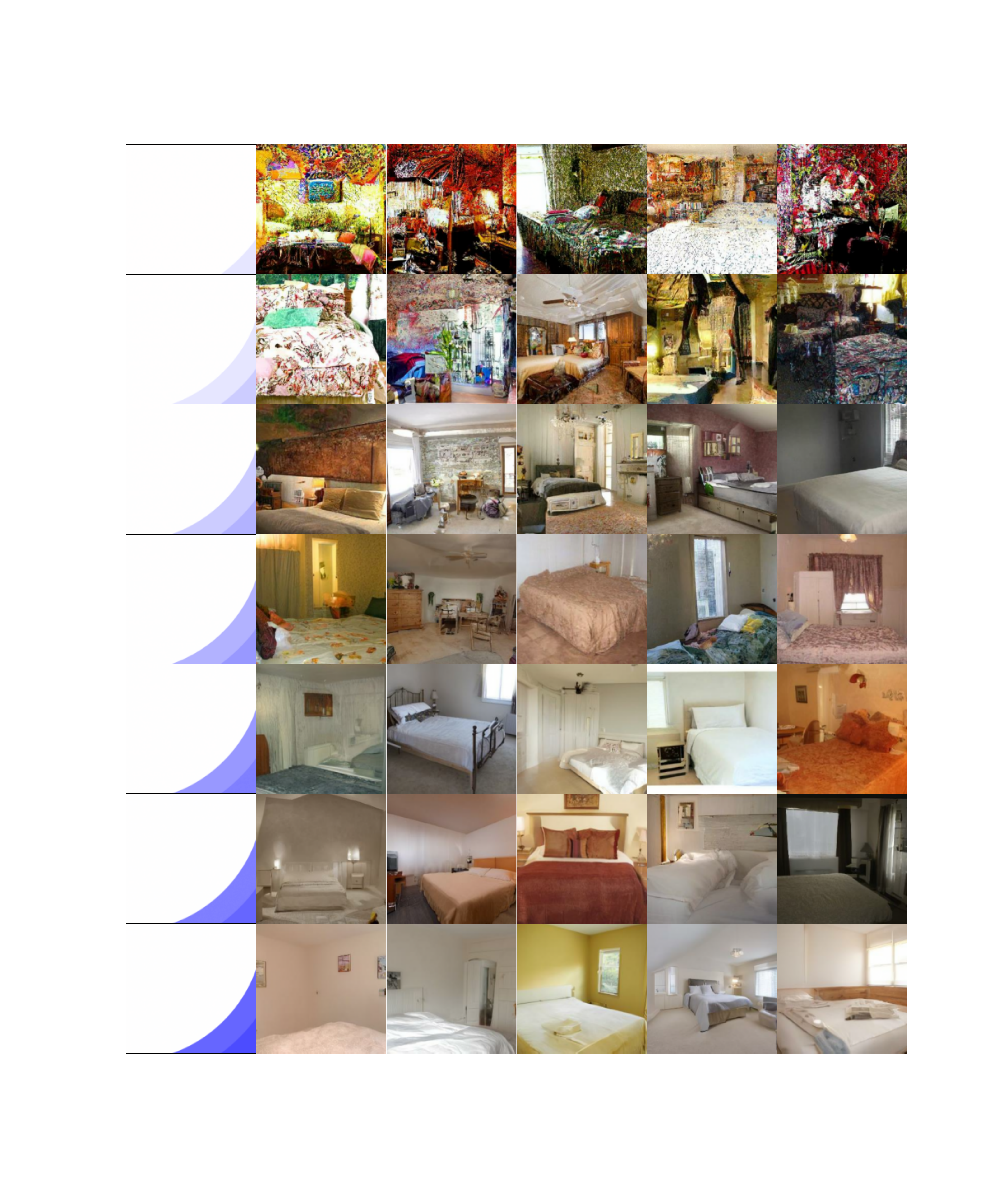}}
    \vspace{-15ex}
	\caption{Samples produces by TDAS with NCSN++~\cite{song2020score} on LSUN (bedroom)~\cite{yu2015lsun}. The first column represents the frequency filter using the same color map as Fig.~\ref{fig:freq_mask}. We fix $(r_{1},r_{2}) = (0.7,0.875)$ and change $(\lambda_1,\lambda_2)$  from top to bottom as follows: $(1,0.9),(0.9,0.8),(0.8,0.7),(0.7,0.6),(0.6,0.5),(0.5,0.4),(0.4,0.3)$.}\label{fig: para1}
	\end{center}
\end{figure}
\begin{figure}[ht]
\vspace{-14ex}
	\begin{center}
    \centerline{\includegraphics[scale=0.42]{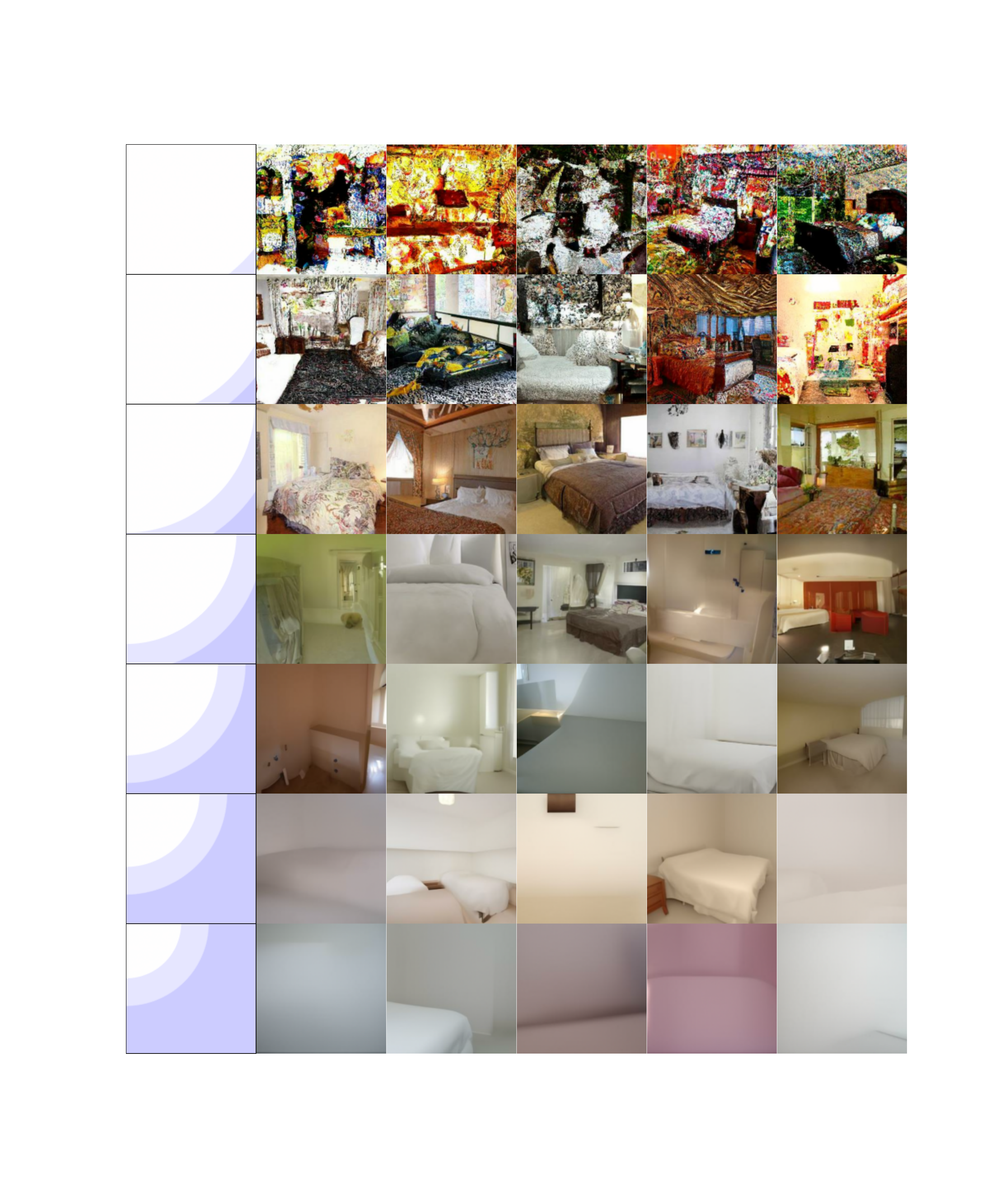}}
     \vspace{-15ex}
    		\caption{Image samples produced by TDAS with NCSN++~\cite{song2020score} on LSUN (bedroom)~\cite{yu2015lsun}. The first column represents the frequency filter using the same color map as Fig.~\ref{fig:freq_mask}. We fix $(\lambda_1,\lambda_2) = (0.9,0.8)$ and change $(r_{1},r_{2})$  from top to bottom as follows: $(0.9,1),(0.8,0.7),(0.7,0.6),(0.6,0.5),(0.5,0.4),(0.4,0.3),(0.3,0.2)$.}\label{fig: para2}
	\end{center}
\end{figure}

\begin{figure}[ht]
    \vspace{-30ex}
    \includegraphics[scale=0.28]{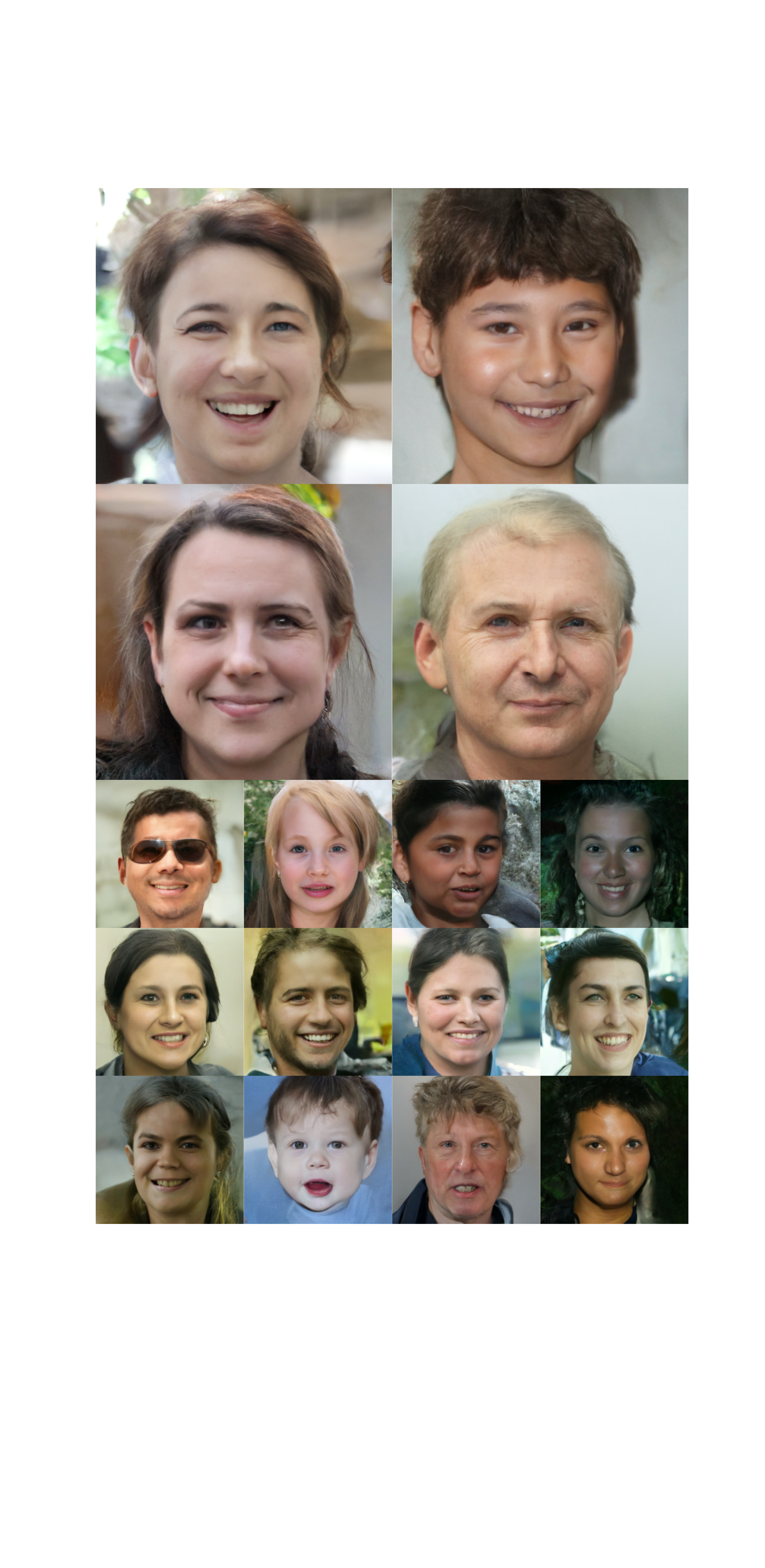}
    \vspace{-41ex}
    \caption{FFHQ ($1024\times1024$)~\cite{karras2019style} samples generated by NCSN++~\cite{song2020score} with TDAS under 100 iterations.}
    \label{fig:demo_sm}
\end{figure}
\begin{figure}[ht]
    \vspace{-50ex}
    \hspace{-16ex}
    \includegraphics[scale=0.25]{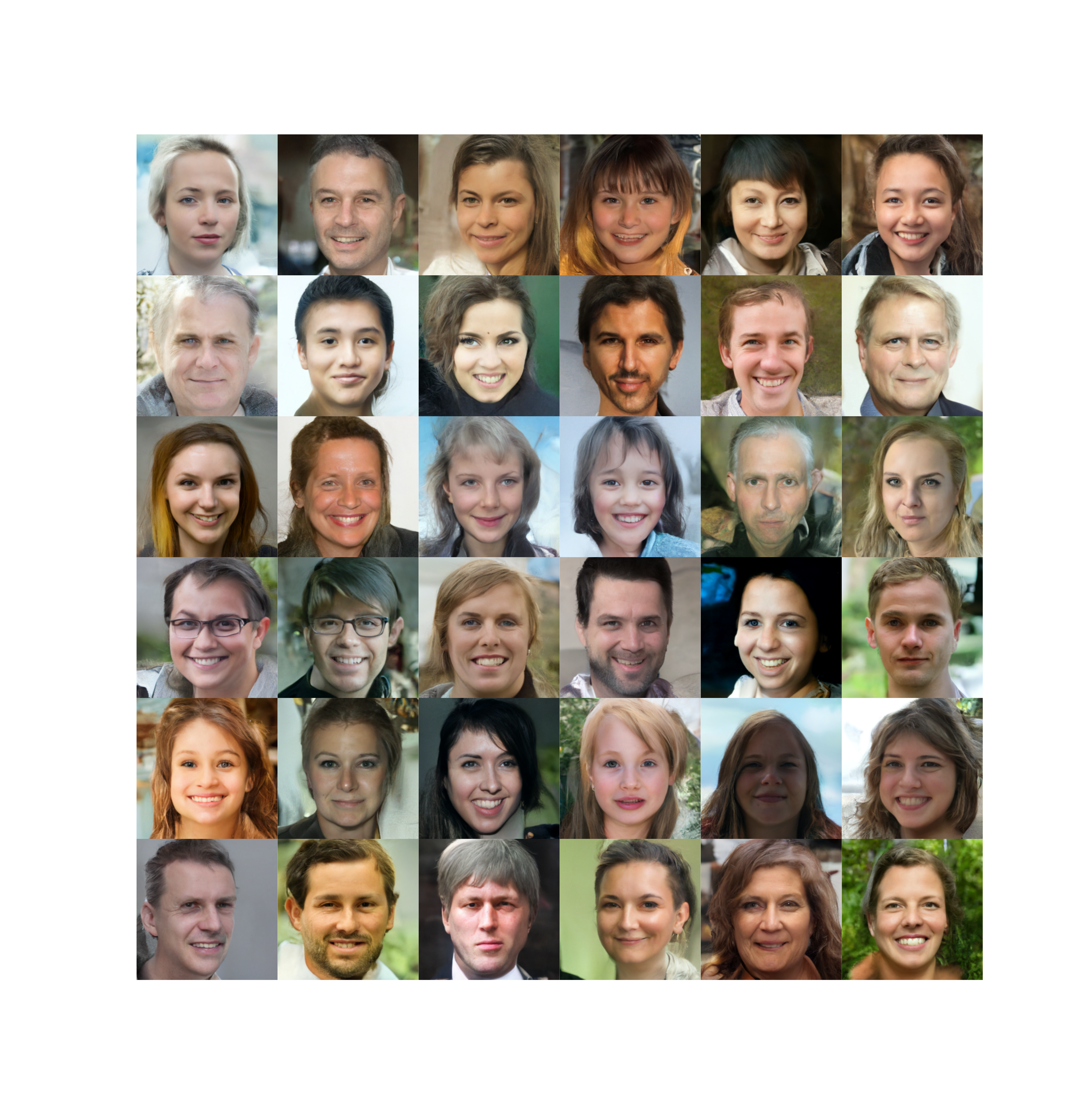}
    \vspace{-15ex}
    \caption{FFHQ ($1024\times1024$)~\cite{karras2019style} samples generated by NCSN++~\cite{song2020score} with TDAS under 100 iterations.}
    \label{fig:demo_sm2}
\end{figure}
\begin{figure}[ht]
    \vspace{-20ex}
    \hspace{-16ex}
    \includegraphics[scale=0.25]{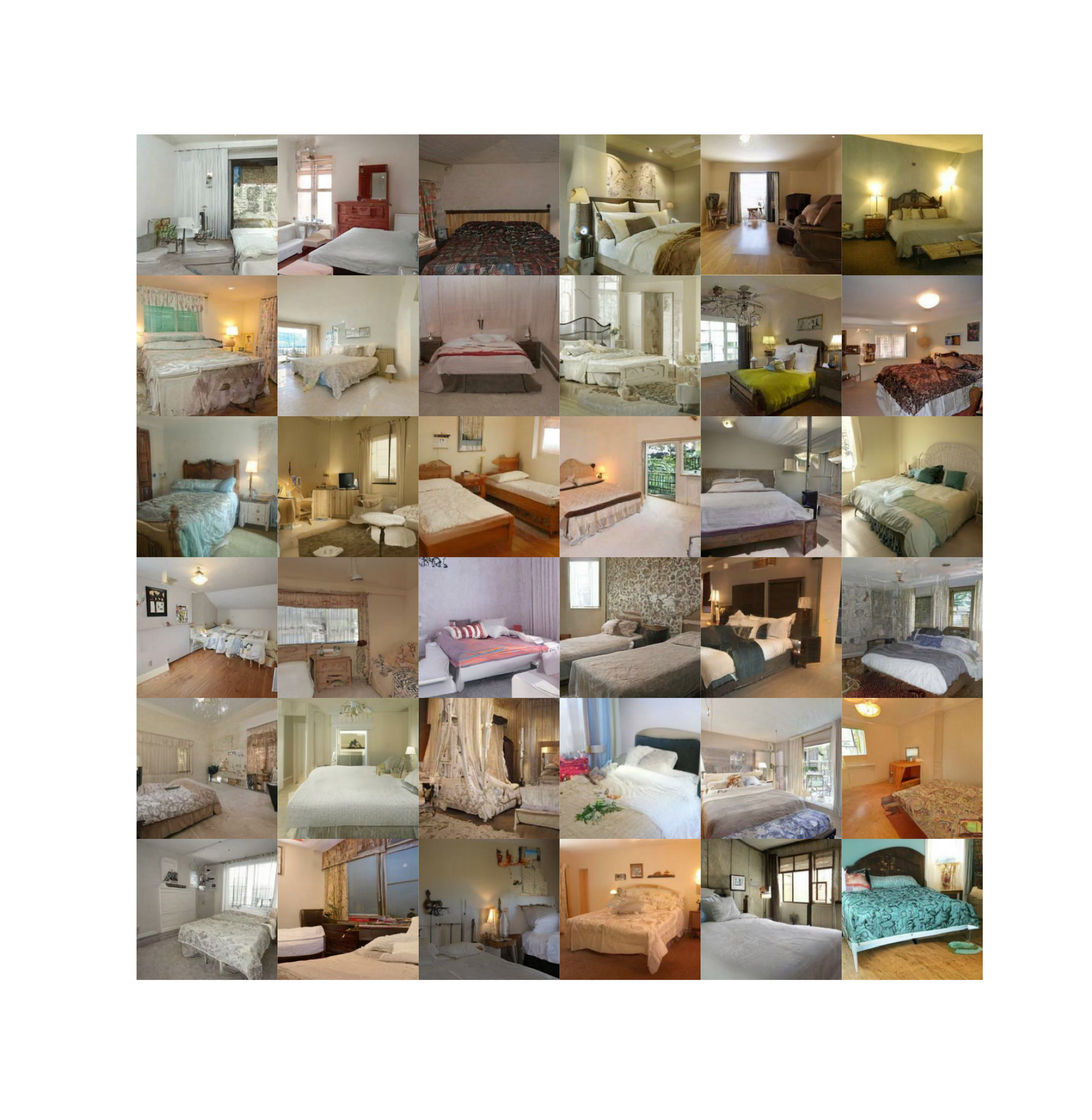}
    \vspace{-15ex}
    \caption{LSUN (bedroom)~\cite{yu2015lsun} samples generated by NCSN++~\cite{song2020score} with TDAS under 400 iterations. Resolution: $256\times256$.}
    \label{fig:bedroom_demo}
\end{figure}
\begin{figure}[ht]
    \vspace{-20ex}
    \hspace{-16ex}
    \includegraphics[scale=0.25]{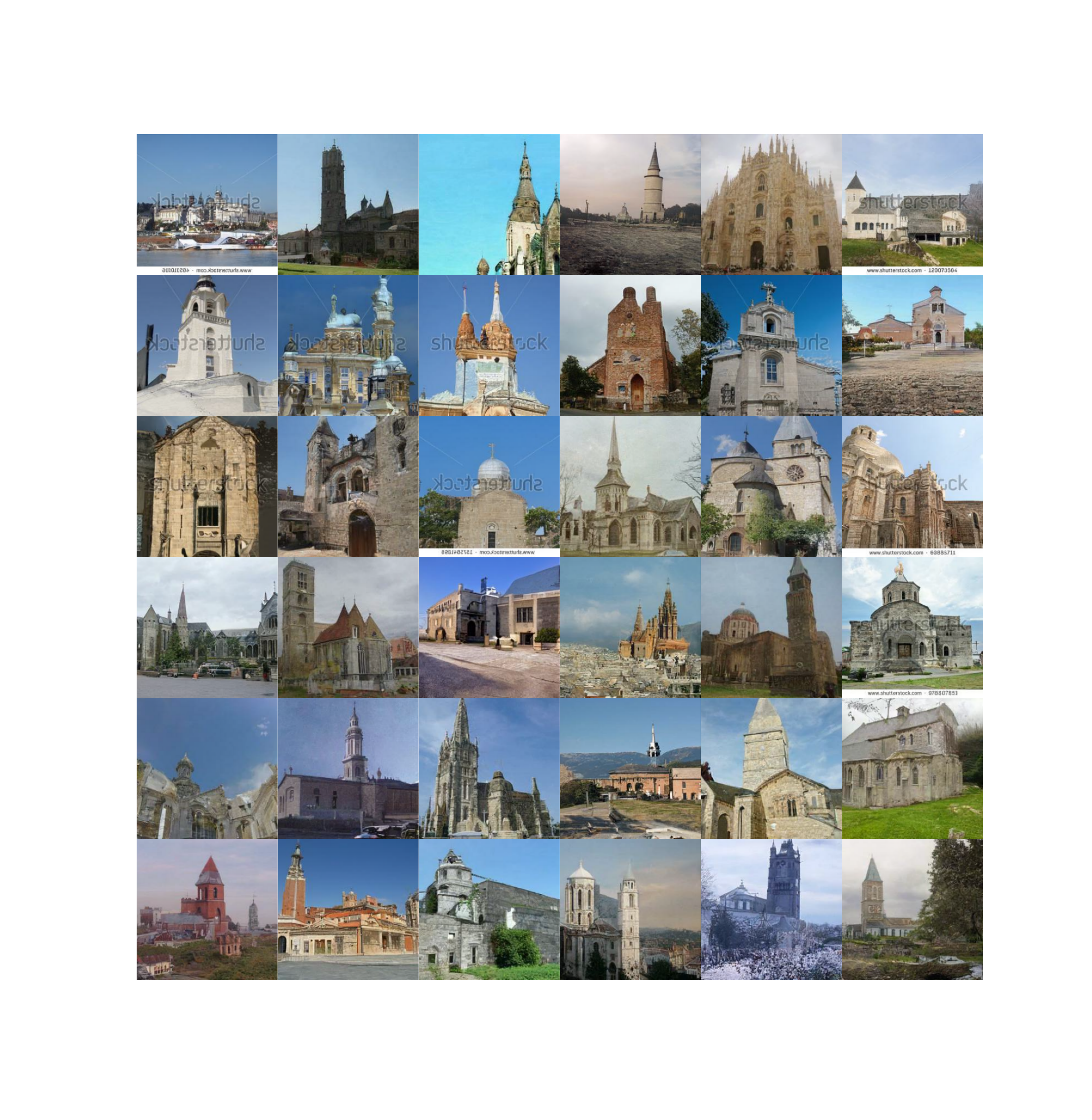}
    \vspace{-15ex}
    \caption{LSUN (church)~\cite{yu2015lsun} samples generated by NCSN++~\cite{song2020score} with TDAS under 400 iterations. Resolution: $256\times256$.}
    \label{fig:church_demo}
\end{figure}

\end{document}